\documentclass[runningheads]{llncs}
\usepackage[T1]{fontenc}

\usepackage{microtype}
\usepackage{graphicx}
\usepackage{enumerate}
\usepackage{subcaption}
\usepackage[ruled,vlined]{algorithm2e}

\usepackage{booktabs} 

\usepackage{hyperref}

\usepackage{mathtools}

\usepackage{amsmath,amssymb,amsfonts}
\usepackage{amsfonts} 
\usepackage{graphicx}
\usepackage{booktabs}
\usepackage[numbers,sort&compress]{natbib}
\usepackage{wrapfig}

\usepackage{ulem}

\usepackage{lipsum}
\usepackage{newfloat}
\usepackage{listings}
\usepackage{stfloats}
\usepackage[capitalize,noabbrev]{cleveref}

\usepackage[table,xcdraw,dvipsnames]{xcolor}
\definecolor{mine}{RGB}{205, 232, 248}

\hypersetup{colorlinks,linkcolor={red},citecolor={green}}

\newtheorem{assumption}{Assumption}
\crefname{assumption}{assumption}{assumptions}
\usepackage[misc]{ifsym}

\usepackage[textsize=tiny]{todonotes}

\begin{document}
\title{Eigensubspace of Temporal-Difference Dynamics and How It Improves Value Approximation in Reinforcement Learning}

\titlerunning{Eigensubspace of TD 
\& How It Improves Value Approximation in RL}
\toctitle{Eigensubspace of Temporal-Difference Dynamics and How It Improves Value Approximation in Reinforcement Learning}
%
\author{Qiang He\inst{1}$^{\text{(\Letter)}}$ \and Tianyi Zhou\inst{2} \and Meng Fang\inst{3} \and Setareh Maghsudi\inst{1}}
\authorrunning{Q. He et. al.}
\tocauthor{Qiang He, Tianyi Zhou, Meng Fang, Setareh Maghsudi}
%
\institute{University of Tübingen, Tübingen, Germany \\\email{\{qiang.he, setareh.maghsudi\}@uni-tuebingen.de} \and  University of Maryland, College Park, USA \\\email{tianyi@umd.edu}\and University of Liverpool, Liverpool, UK\\\email{Meng.Fang@liverpool.ac.uk}}
\maketitle              
%




\begin{abstract}

We propose a novel value approximation method, namely ``\underline{E}igensubspace \underline{R}egularized \underline{C}ritic (ERC)'' for deep reinforcement learning (RL). ERC is motivated by an analysis of the dynamics of Q-value approximation error in the Temporal-Difference (TD) method, which follows a path defined by the 1-eigensubspace of the transition kernel associated with the Markov Decision Process (MDP). It reveals a fundamental property of TD learning that has remained unused in previous deep RL approaches. In ERC, we propose a regularizer that guides the approximation error tending towards the 1-eigensubspace, resulting in a more efficient and stable path of value approximation. Moreover, we theoretically prove the convergence of the ERC method. Besides, theoretical analysis and experiments demonstrate that ERC effectively reduces the variance of value functions. Among 26 tasks in the DMControl benchmark, ERC outperforms state-of-the-art methods for 20. Besides, it shows significant advantages in Q-value approximation and variance reduction. Our code is available at~\href{https://sites.google.com/view/erc-ecml23/}{https://sites.google.com/view/erc-ecml23/}.

\end{abstract}
\section{Introduction}

In recent years, deep reinforcement learning (RL), which is built upon the basis of the Markov decision process (MDP), has achieved remarkable success in a wide range of sequential decision-making tasks~\citep{rl}, including board games \citep{alpha}, video games~\citep{nature_dqn, starcraft}, and robotics manipulation~\citep{sac}. Leveraging the rich structural information in MDP, such as the Markov assumption~\citep{rl}, the stochasticity of transition kernel \citep{agarwal2019reinforcement}, and low-rank MDP~\citep{low-rank-1, lowrank2, low-rank-3, low-rank-4}, supports designing efficient RL algorithms. 

\begin{figure}
\centering
\includegraphics[width=1.0\textwidth]{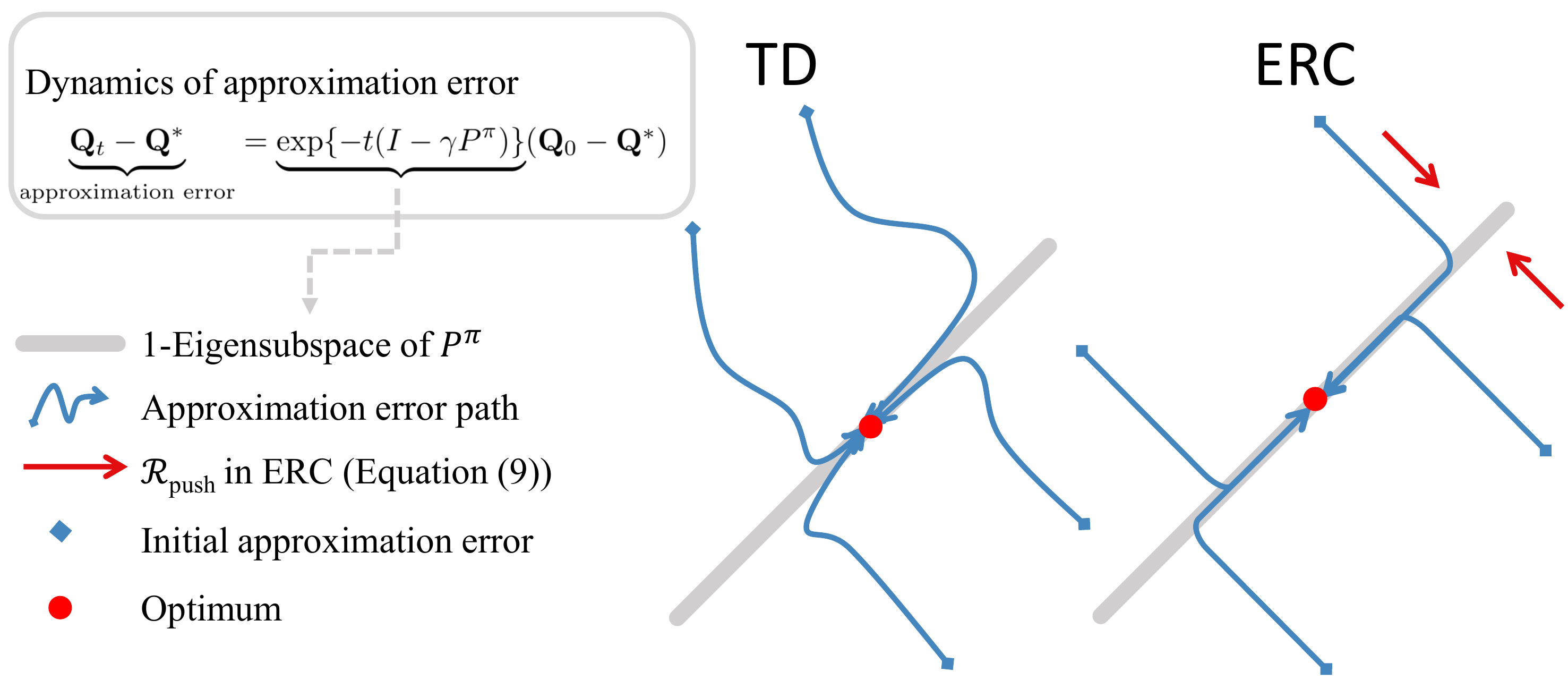}
\caption{\label{fig: dynamics}Value approximation error ($Q - Q^*$) path for TD and our proposed ERC algorithms. The approximation error of the TD method gradually approaches the 1-eigensubspace (as defined in \cref{remark: dynamics of td method}) before ultimately converging to the optimum. This path is referred to as the inherent path. ERC leverages this inherent path by directly pushing the approximation error towards the 1-eigensubspace through $\mathcal{R}_{\text{push}}$ (defined in \cref{eq: PE regularizer}), resulting in a more efficient and stable learning process.}
\end{figure}

Motivated by the potential benefits of utilizing structural information to design efficient RL algorithms \citep{lyle2021understanding,lyle2021effect, lylelearning,he2023frustratingly,YangAA22}, we investigate the dynamics of Q value approximation induced by the Temporal-Difference (TD) method. The Q value, commonly used to design DRL algorithms, is analyzed in the context of the matrix form of the Bellman equation. We study continuous learning dynamics. We also examine the crucial role of the transition kernel of MDP in the dynamics of the Bellman equation. We perform an eigenvalue decomposition of that dynamic process, which reveals that for any MDP, the TD method induces an inherent learning path of approximation error in the value approximation process.

Furthermore, as \cref{fig: dynamics} shows, the approximation error of the TD method is optimized towards the 1-eigensubspace before ultimately converging to zero. It is in direct contrast with the optimum that the Monte Carlo (MC) method achieves \citep{rl}. Thus, such an inherent path is non-trivial. Such a path, which corresponds to prior knowledge of MDP, has been neglected in designing DRL algorithms \citep{lyle2021understanding,lyle2021effect, lylelearning,YangAA22}, despite its great potential. Our main idea is utilizing the prior knowledge to improve the efficiency of value approximation by directly guiding the approximation error towards the 1-eigensubspace, leading to a more efficient and stable path. To that end, we propose a novel value approximation method, \underline{E}igensubspace \underline{R}egularized \underline{C}ritic (ERC), as shown in \cref{fig: dynamics}. We also establish the convergence of our proposal. Theoretical analysis and experiments demonstrate that ERC effectively reduces the variance of value functions.

To evaluate the effectiveness of our proposed algorithm, ERC, we conduct extensive experiments on the continuous control suite DMControl~\citep{dm-control}. Our empirical results demonstrate that ERC performs well in terms of approximation error. Moreover, by examining the ERC variance reduction, we verify its superior performance compared to the algorithms specifically designed for variance control (e.g., TQC~\citep{tqc}, REDQ~\citep{redq}), as consistent with our theoretical analysis. All in all, comparisons show that ERC outperforms the majority of the state-of-the-art methods (\textbf{20} out of \textbf{26} tasks) while achieving a similar performance otherwise. Specifically, on average, the ERC algorithm surpasses TQC, REDQ, SAC~\citep{sac}, and TD3~\citep{td3} by 13\%, 25.6\%, 16.6\%, and 27.9\%, respectively.

The main contributions of this work include i) identifying the existence of an inherent path of the approximation error for TD learning, ii) leveraging the inherent path to introduce a more efficient and stable method, ERC, with a convergence guarantee, and iii) demonstrating, through comparison with state-of-the-art deep RL methods, the superiority of ERC in terms of performance, variance, and approximation error.
\section{Preliminaries}
To formalize RL, one uses an MDP framework consisting of 6-tuple $(\mathcal{S, A, }R, P, \gamma, \rho_0)$, where $\mathcal{S}$ denotes a state space, $\mathcal{A}$ an action space, $R: \mathcal{S} \times \mathcal{A} \rightarrow \mathbb{R}$ a reward function, $P: \mathcal{S} \times \mathcal{A} \rightarrow p(s)$ a transition kernel, $\gamma \in [0, 1)$ a discount factor, and $\rho_0$ an initial state distribution.

Deep RL focuses on optimizing the policy through return, defined as $R_t=\sum_{i=t}^{T}\gamma^{i-t} r(s_i, a_i)$. The action value (Q) function, $Q^{\pi}(s, a)$, represents the quality of a specific action, $a$, in a state, $s$, for a given policy $\pi$. Formally, the Q function is defined as
\begin{equation}
 Q^\pi(s,a) = \mathbb{E}_{\tau \sim \pi, p} [R_{\tau} | s_0 = s, a_0 = a],
\end{equation}
where $\tau$ is a state-action sequence $(s_0, a_0, s_1, a_1, s_2, a_2 \cdots)$ induced by a policy $\pi$ and $P$. The state value (V) function is $V^\pi(s) = \mathbb{E}_{\tau \sim \pi, p} [R_{\tau} | s_0 = s]$. A four-tuple $(s_t, a_t, r_t, s_{t+1})$ is referred to as a transition. The $Q$ value can be recursively computed by Bellman equation \cite{rl}
\begin{equation}
Q^\pi(s,a) = r(s,a) + \gamma \mathbb{E}_{s',a'} [Q^\pi(s',a')],
\label{eq: bellman consistency equation}
\end{equation}
where $s' \sim p(\cdot|s,a)$ and $a \sim \pi(\cdot |s)$. 
The process of using a function approximator (e.g. neural networks) to estimate Q or V values is referred to as value approximation.

\textbf{Bellman equation in matrix form.} 
Let $\mathbf{Q}^\pi$ denote the vector of all Q value with length $|\mathcal{S}|\cdot |\mathcal{A}|$, and $\mathbf{r}$ as vectors of the same length. We overload notation and let P refer to a matrix of dimension $(|\mathcal{S}|\cdot |\mathcal{A}|) \times |\mathcal{S}|$, with entry $P_{s,a,s'}$ equal to $P(s'|s,a)$. We define $P^\pi$ to be the transition matrix on state-action pairs induced by a stationary policy $\pi$
\begin{equation}\label{eq: definition of transition matrix induced by a policy}
    P^{\pi}_{s,a,s',a'} := P(s'|s,a)\pi (a'|s').
\end{equation}
Then, it is straightforward to verify 

\begin{equation}\label{eq: bellman eq matrix form}
    \mathbf{Q}^\pi = \mathbf{r} + \gamma P^\pi \mathbf{Q}^\pi,
\end{equation}
where $P^\pi \in \mathbb{R}^{|\mathcal{S}|\cdot|\mathcal{A} |\times  |\mathcal{S}|\cdot|\mathcal{A}|}$. 
The following famous eigenpair result holds.
 \begin{remark}[Eigenpair for stochastic matrix $P^\pi$~\citep{matrix-analysis}]
 \label{remark: Eigenpair for stochastic matrix}
The spectral radius of $P^\pi$ is 1. The eigenvector corresponding to 1 is $\mathbf{e}$, where $\mathbf{e}$ is a column of all 1’s.  
\end{remark}

\section{Method}
In this section, we start with the dynamics of Q value approximation induced by the TD method. This analysis reveals that in the value approximation, the approximation error has an inherent path. We leverage that path to design our proposed practical algorithm, ERC. We also investigate the convergence property of our method. All proofs are available in the appendix.

\subsection{An Inherent Path of Value Approximation}
\label{sec: dynamics of bellman error}
Motivated by the potential benefits of using structural information to design novel RL algorithms \citep{lyle2021understanding,lyle2021effect, lylelearning,poer,YangAA22}, we examine the dynamics of Q value induced by the Bellman equation.
Given \cref{eq: bellman eq matrix form} in matrix form, the true Q function of policy $\pi$ can be directly solved.
\begin{remark}[True Q value of a policy \citep{agarwal2019reinforcement}]
\label{corollary: Q pi}
Given a transition matrix induced by a policy $\pi$, as defined in \cref{eq: definition of transition matrix induced by a policy}, the true Q function of policy $\pi$ can be directly solved by
\begin{equation}
    \mathbf{Q}^*=\left(I-\gamma P^\pi\right)^{-1} \mathbf{r}, 
\end{equation}
where $I$ is the identity matrix and $\mathbf{Q}^*$ is the true Q function of policy $\pi$.
\end{remark}

To simplify the notation, let $X = (s, a)$.
The one-step temporal difference (TD) continuous learning dynamics follows as
\begin{equation}
    \partial_t Q_t (x) = \mathbb{E}_\pi [r_t + \gamma Q_t(x_{t+1}) | x_t] -Q_t(x).
\end{equation}
According to \cref{eq: bellman eq matrix form}, we have the matrix form
\begin{equation} 
\label{eq: dynamics differential form}
    \partial_t \mathbf{Q}_t = - (I - \gamma P^\pi) \mathbf{Q}_t + \mathbf{r}.
\end{equation}
\Cref{eq: dynamics differential form} is a differential equation that is directly solvable using \cref{corollary: Q pi}.
\begin{lemma}[Dynamics of approximation error] \label{thm: dynamics of Q}
Consider a continuous sequence $\{Q_t | t \geq 0 \}$, satisfy \cref{eq: dynamics differential form} with initial condition $\mathbf{Q}_0$ at time step $t=0$, then
\begin{equation}
\label{eq: dynamics of approximation error}
\mathbf{Q}_t - \mathbf{Q}^* = \exp \{ -t (I - \gamma P^\pi)   \} (\mathbf{Q}_0 - \mathbf{Q}^*) .
\end{equation}
\end{lemma}
From \cref{thm: dynamics of Q}, $\mathbf{Q}_t$ converges to $\mathbf{Q}^*$, as $t \to \infty$. The approximation error, $\mathbf{Q}_t - \mathbf{Q}^*$, appears in \cref{eq: dynamics of approximation error}, which reveals the dynamics of approximation error is related to the structure of transition kernel and policy $\pi$. Moreover, there is rich structural information in $P^\pi$, which inspires us to consider \cref{eq: dynamics of approximation error} in a more fine-grained way. To better understand the approximation error, following \citet{ghosh2020representations, lyle2021effect, lylelearning}, we make the following assumption.
\begin{assumption}
\label{assumption: P pi diagnoalize}
    $P^\pi$ is a real-diagonalizable matrix with a strictly decreasing eigenvalue sequence $ |\lambda_1|, |\lambda_2|, \cdots, |\lambda_{|\mathcal{S}|\cdot |\mathcal{A}|} |$, and the corresponding eigenvector $H_1, H_2 , \cdots, H_{|\mathcal{S}|\cdot |\mathcal{A}|}$.
\end{assumption}
\Cref{remark: Eigenpair for stochastic matrix} shows that $\lambda_1 = 1 $ because the $P^\pi$ is a stochastic matrix, and the eigenvector corresponding to 1 is $\mathbf{e}$. That inspires us to perform an eigenvalue decomposition for \cref{eq: dynamics of approximation error}. 
\begin{theorem}
\label{theorem: path of td method}
If \Cref{assumption: P pi diagnoalize} holds,
we have $\mathbf{Q}_t - \mathbf{Q}^* = \alpha_1 \exp\{ t(\gamma \lambda_1 -1 )\}H_1 + \sum_{i=2}^{|\mathcal{S}|\cdot |\mathcal{A}|} \alpha_i \exp \{ t(\gamma \lambda_i -1) \} H_i = \alpha_1 \exp\{ t(\gamma \lambda_1 -1) \}H_1 + o\Big( \alpha_1 \exp\{ t(\gamma \lambda_1 -1 )\} \Big)$, where $\alpha_i$ is a constant.
\end{theorem}
\cref{theorem: path of td method} states that the approximation error of the TD method can be decomposed into two components. One of the components is primarily influenced by the 1-eigensubspace. Thus, the approximation error of the value function in the TD method follows a path induced by the 1-eigensubspace, which is a result of the stochasticity inherent in MDP.
\begin{remark}[An inherent path of TD method]
\label{remark: dynamics of td method}
   Given an MDP consisting of 6-tuple $(\mathcal{S, A, }R, P, \gamma, \rho_0)$, policy $\pi$, and a Banach space $(\mathcal{Q}, \| \cdot \|)$, there exists an inherent path that the Bellman approximation error, i.e., $\mathbf{Q}_t - \mathbf{Q}^*$, starts at the initial point and then approaches the 1-eigensubspace before ultimately converging to zero. The 1-eigensubspace, which is induced by $P^\pi$, is defined as $\{c \: \mathbf{e}\}$, where $c \in \mathbb{R}$ and $\mathbf{e}$ is a column of all ones. 
\end{remark}

\begin{wrapfigure}{r}{0.5\textwidth}
  \centering
    \vspace{-20pt}
\includegraphics[width=0.5\textwidth]{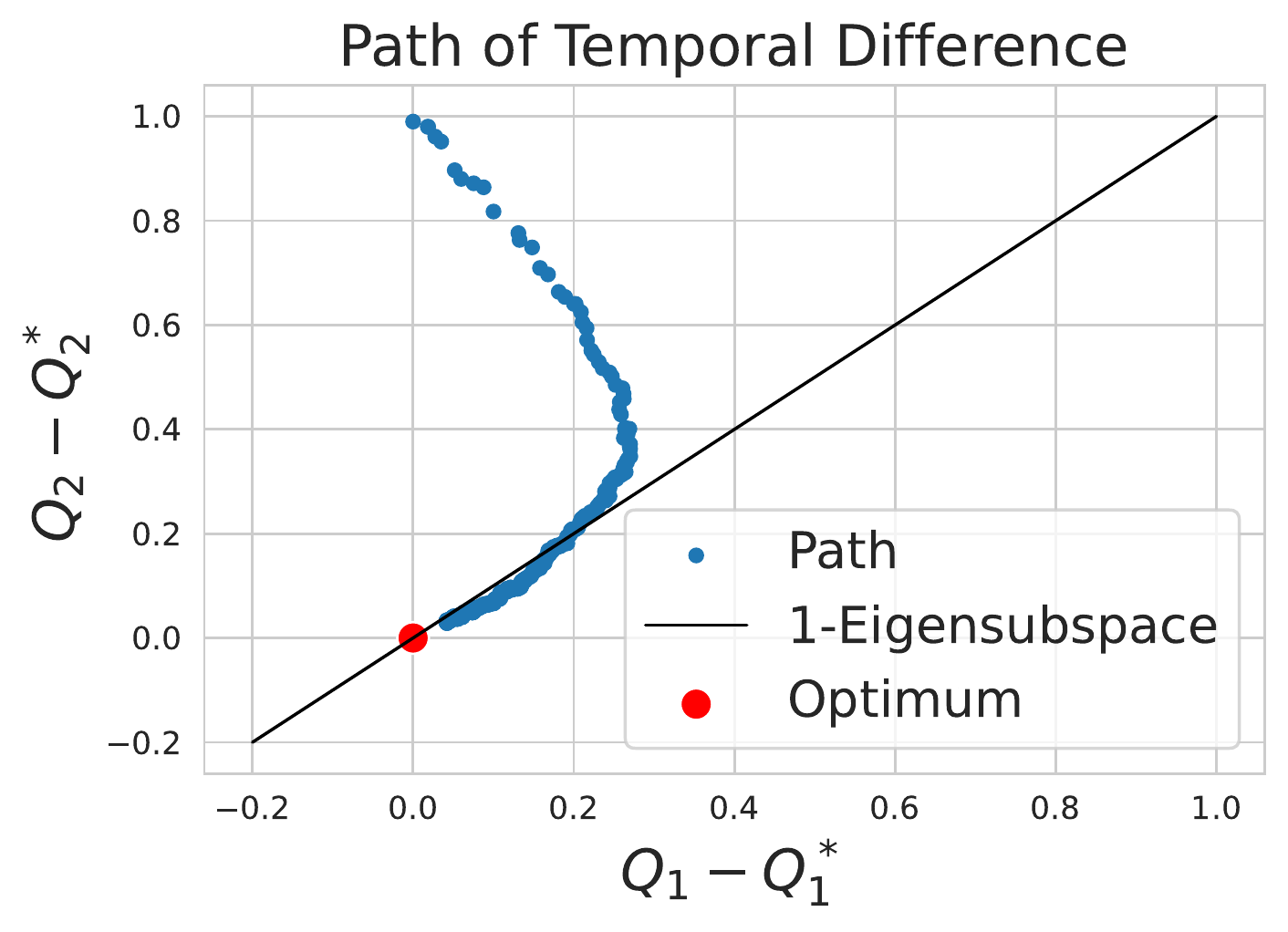}
  \caption{\label{fig:  dynamics td}Path of TD method.  There exists an inherent path that approximation error approaches 1-eigensubspace before converging to zero. The empirical fact is consistent with our theoretical analysis in \cref{theorem: path of td method}.}
  \vspace{-15pt}
\end{wrapfigure}


\textbf{What does a theoretically inherent path really look like in practice?} The aforementioned content discusses the inherent path of approximation error. Does this occur in the practical scene? To empirically show the path of the approximation error, we visualize the path of approximation error given a fixed policy. The results are given in \cref{fig:  dynamics td}, where we perform experiments on \href{https://github.com/openai/gym/blob/master/gym/envs/toy_text/frozen_lake.py}{FrozenLake-v1 environment} since the true Q value $Q^*$ of this environment can be evaluated by the Monte Carlo Method. The approximation error of the TD method is optimized toward the 1-eigensubspace before ultimately converging to zero rather than directly toward the optimum. The inherent path in \cref{fig:  dynamics td} is consistent with \cref{remark: dynamics of td method}. Thus, such a path is non-trivial and motivates us to improve the value approximation by guiding the approximation error to 1-eigensubspace, resulting in an efficient and robust path. 

\subsection{Using Eigensubspace Regularization to Improve Value Approximation }
\label{sec: Improve value approximation via Eigensubspace regularization}
%
\begin{figure}[!t]
\centering
\subcaptionbox{\label{fig: sub distance2eigensubspace}}
{\includegraphics[width=0.5\textwidth]{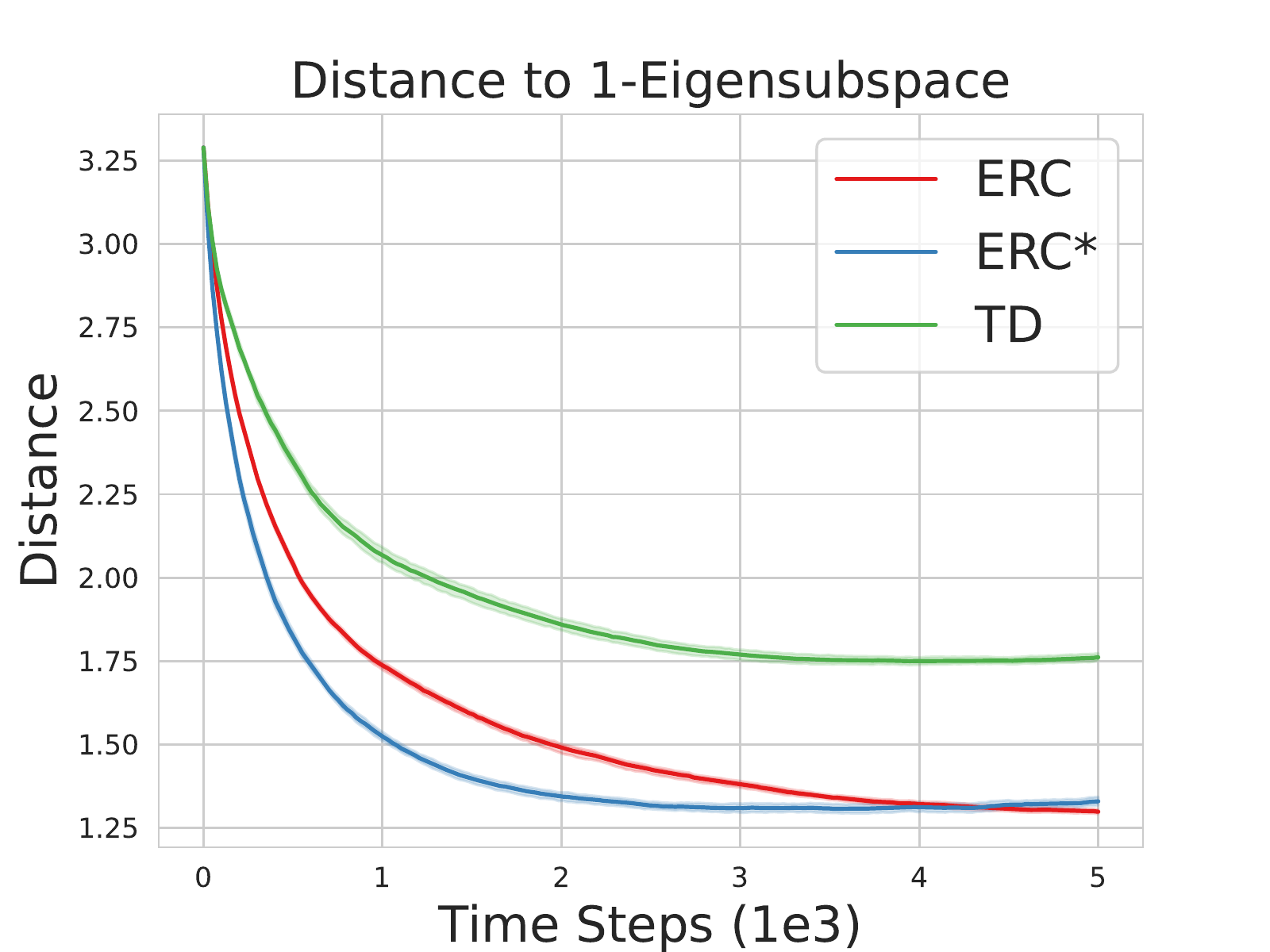}}
\hspace{-0.2in}
\subcaptionbox{\label{fig: sub Approximation_Error}}
{\includegraphics[width=0.5\textwidth]{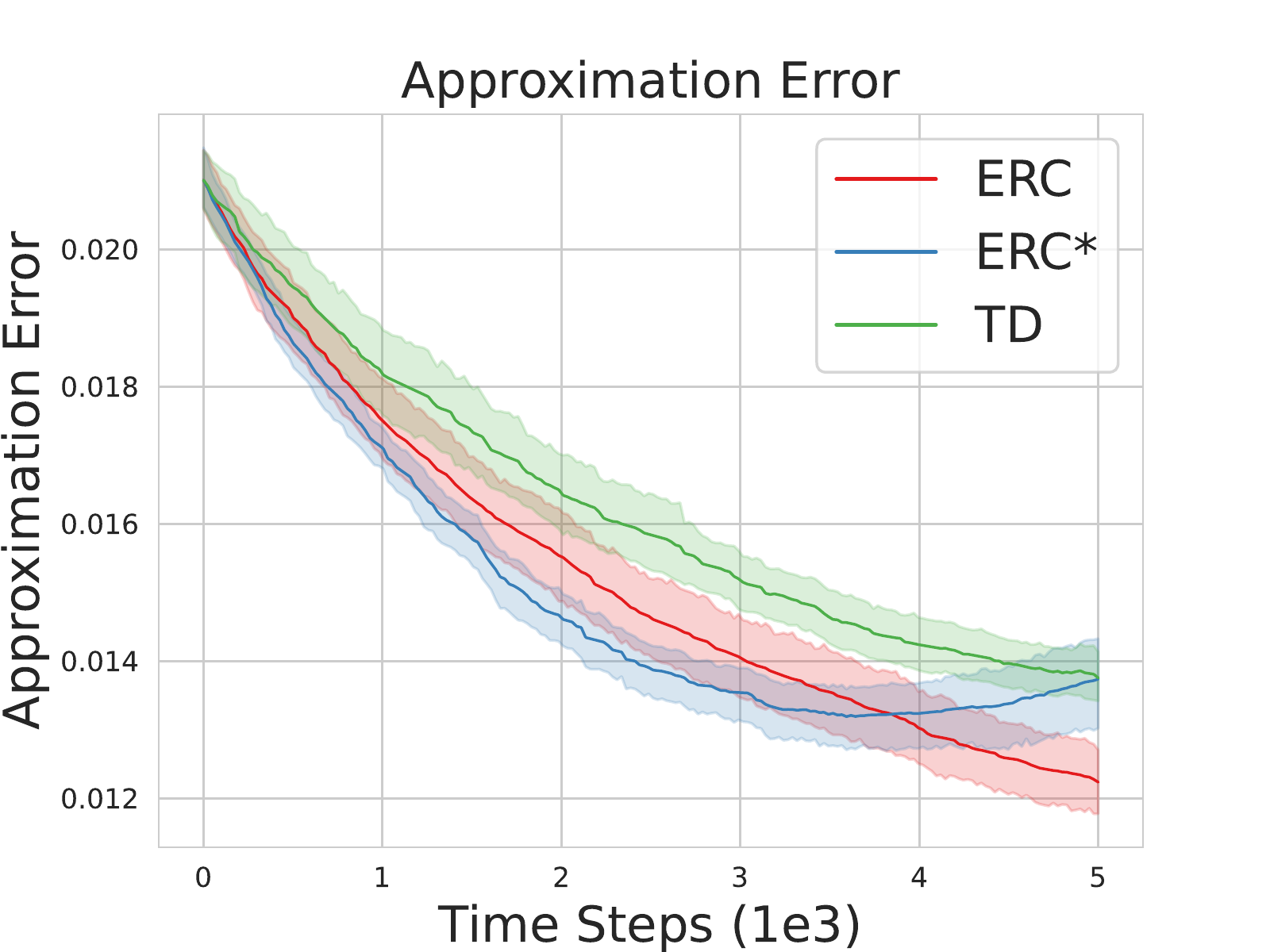}}
\caption{\label{fig: distance 2 eigensubspace and approximation error}Value function approximation process for various methods on \href{https://github.com/openai/gym/blob/master/gym/envs/toy_text/frozen_lake.py}{FrozenLake-v1} environment. (a) illustrates the distance between the approximation error and the 1-eigensubspace for various methods, where ERC* denotes ERC utilizing an oracle true Q value, $Q^*$, to push the approximation error towards the 1-eigensubspace. The results demonstrate that the ERC method is closer to the 1-eigensubspace at the same time compared to the TD method. (b) represents the absolute approximation error for various algorithms. The result illustrates that the ERC method has a smaller approximation error at the same time than the TD method. For both metrics, ERC is observed to outperform or be at least as good as ERC* in later stages. The shaded area represents a standard deviation over ten trials.}
\end{figure} 
The inherent path, which can be viewed as prior knowledge of MDP, has not been utilized to design DRL algorithms in previous work~\citep{lyle2021understanding,lyle2021effect, lylelearning,YangAA22}. We improve the value approximation by directly guiding the approximation error towards the 1-eigensubspace, leading to a more efficient and stable path.

The true Q value, $Q^*$, is always unknown when designing DRL algorithms. However, we can use a target Q as an approximate substitute for $Q^*$ for two reasons. Firstly, from the perspective of value function optimization, the objective of the Bellman equation optimization is to make the learned Q-value as close as possible to $Q^*$. That is achievable by minimizing the distance between the learned Q and the target Q. Instead of using $Q^*$ directly in learning, the target Q is used to approximate $Q^*$ through a bootstrap approach. Similarly, we use target Q in our ERC algorithm design. Secondly, in our experiments, we find that using target Q to replace $Q^*$ produces an effect that approximates the effect produced by using $Q^*$, as illustrated in \cref{fig: distance 2 eigensubspace and approximation error}. We calculate the distance of the approximation error to the 1-eigensubspace and approximation error during the optimization of the value function, and we can see that: i) the ERC algorithm using target Q instead of $Q^*$ allows the approximation error to reach the 1-eigensubspace faster than the TD algorithm and the effect can be approximated to that of the ERC using $Q^*$ (ERC* algorithm). And ii) for the approximation error, the ERC algorithm obtains a relatively small approximation error compared to the TD method. The approximation error of ERC is smaller than that of ERC* in the later optimization stages. Therefore, it is reasonable to replace $Q^*$ with the target Q in the practical design of the algorithm. Using target Q-value is a feasible solution that allows us to achieve results similar or better to using $Q^*$, and it has the advantage of being more easily implemented in practice. Thus, the Bellman error is pushed to 1-eigensubspace in the ERC algorithm.

To push the Bellman error toward the 1-eigensubspace, it is essential to project the error onto that subspace. Therefore, one must determine the projected point in the 1-eigensubspace to the Bellman error. 
\begin{lemma}\label{lemma: distance to 1-eigensubspace}
    Consider a Banach space $(\mathfrak{B}, \| \cdot \|)$ of dimension N, and let the N-dimensional Bellman error at timestep $t$, represented by $\mathbf{B}^t$, have coordinates $(B_1, B_2, \cdots, B_N)$ in $(\mathfrak{B}, \| \cdot \|)$. Within this Banach space, the projected point in the 1-eigensubspace, which is closest to $B^t$, is $Z^t$ whose coordinates are $(z^t, z^t,\cdots, z^t)$, where $z^t = \frac{1}{N}\sum_{j=1}^{N} B_i^t$.
\end{lemma}

The Bellman error can be pushed towards 1-eigensubspace at each timestep with the help of \cref{lemma: distance to 1-eigensubspace}. To accomplish this, we minimize the regularization term
\begin{equation}
\label{eq: PE regularizer}
    \mathcal{R}_{\text{push}} (\theta) = \frac{1}{N}\sum_{i=1}^{N}\| B_i - Z \|_2^2,
\end{equation}
where $B_i$ represents the Bellman error at the $i$-dimension and $Z = \frac{1}{N}\sum_{j=1}^{N} B_i(\theta) $. 
By combining \cref{eq: PE regularizer} with policy evaluation loss $\mathcal{L}_{\textbf{PE}}$, the  \underline{E}igensubspace \underline{R}egularized \underline{C}ritic (ERC) algorithm is defined as
\begin{equation}
\label{eq: ERC loss network form}
	\mathcal{L}_{\text{ERC}}(\theta) = \mathcal{L}_{\text{PE}}(\theta) + \beta \mathcal{R}_{\text{push}}(\theta),
\end{equation}
where $\mathcal{L}_{\text{PE}}(\theta)$ is a policy evaluation phase loss such as  
\begin{equation*}
	\mathcal{L}_{\text{PE}} (\theta) = \big[Q(s,a; \theta) - \big(r(s,a) + \gamma \mathbb{E}_{s',a'} \big[Q(s',a'; \theta')\big] \big) \big]^2,
\end{equation*}
and $\beta$ is a hyper-parameter that controls the degree to which the Bellman error is pushed toward the 1-eigensubspace. ERC enhances the value approximation by pushing the Bellman error to 1-eigensubspace, leading to a more stable and efficient value approximation path. To evaluate the effectiveness of ERC, a case study is conducted on the FrozenLake environment. The distance between the approximation error and the 1-eigensubspace, as well as the absolute approximation error, are used as metrics to compare the performance of ERC with that of the TD method and that of ERC* (ERC utilizing oracle $Q^*$). The results in  \cref{fig: distance 2 eigensubspace and approximation error} demonstrate that ERC is superior to the TD method, and is either superior to or at least as effective as ERC*.

The theoretical benefit of the ERC method can be understood by examining \cref{eq: PE regularizer}. Minimizing $\mathcal{R}_{\text{push}}$ explicitly reduces the variance of the Bellman error. This reduction in variance leads to two benefits: minimizing the variance of the Q-value and minimizing the variance of the target Q-value. That is observable by rewriting \cref{eq: PE regularizer} as
\begin{equation}
\label{eq: R push expanding}
\begin{aligned}
        \mathcal{R}_{\text{push}} (\theta) = & \mathbb{E} \Big( (Q-\mathcal{B}Q)- \mathbb{E} [Q - \mathcal{B}Q]  \Big)^2 \\
= & \underbrace{\mathbb{E}[ (Q- \mathbb{E} [Q])^2]}_{\text{variance of }Q } + \underbrace{\mathbb{E}[( \mathcal{B} Q - \mathbb{E} \mathcal{B}Q)^2]}_{\text{variance of }\mathcal{B}Q } - \\
& 2 \underbrace{\mathbb{E}(Q- \mathbb{E} [Q])( \mathcal{B} Q - \mathbb{E} \mathcal{B}Q )}_{\text{covariance between $Q$ and $\mathcal{B}Q$} },
\end{aligned}
\end{equation}

where $\mathcal{B}$ is a bellman backup operator and $\mathcal{B}Q$ is a target Q. These facts highlight the benefits of the ERC algorithm, as it leads to a more stable and efficient Q value approximation.

\begin{algorithm}[!ht]
\SetAlgoLined
\textbf{Initialize} actor network $\pi$, and critic network $Q$ with random parameters; \\
\textbf{Initialize} target networks and replay buffer $\mathcal{D}$; \\
\textbf{Initialize} $\beta$, total steps $T$, and $t=0$; \\
Reset the environment and receive the initial state $s$; \\
\While{$t < T$}{
Select action w.r.t. its policy $\pi$ and receive reward $r$, new state $s'$; \\
Store transition tuple $(s, a, r, s')$ to $\mathcal{D}$; \\
Sample $N$ transitions $(s, a, r, s')$ from $\mathcal{D}$; \\
Compute $\hat{\mathcal{R}}_{\text{push}}$ by \cref{eq: ERC beta,eq: PE regularizer}; \\
Update critic by minimizing \cref{eq: sac-ERC Q objective}; \\
Update actor by minimizing \cref{eq: sac-ERC policy objective}; \\
Update $\alpha$ by minimizing \cref{eq: sac-ERC alpha objective}; \\
Update target networks; \\
$t\leftarrow t+1$, $s\leftarrow s'$;
}
\caption{ERC (based on SAC~\cite{sac})}
\label{alg:ERC}
\end{algorithm}

\subsection{Theoretical Analysis}
\label{sec: theoretical analysis}
We are also interested in examining the convergence property of ERC. To do this, we first obtain the new Q value after one step of updating in tabular form.

\begin{lemma}\label{lemma: ERC in tabular form}
Given the ERC update rules in \cref{eq: ERC loss network form}, ERC updates the value function in tabular form in the following way
\begin{equation}\label{eq: ERC update rule tabular setting}
\begin{aligned}
     Q_{t+1} = (1- \alpha_t (1+\beta) )  Q_t + \alpha_t (1+\beta) \mathcal{B}Q_t -  \alpha_t\beta C_t,
\end{aligned}
\end{equation}
where $\mathcal{B}Q_t(s_t,a_t) = r_t + \gamma \mathbb{E}_{s_{t+1},a_{t+1}} Q_t(s_{t+1}, a_{t+1})$, $C_t = 2 \mathbb{NG}(\mathbb{E} [ \mathcal{B}Q_t - Q_t ])$, and $\mathbb{NG}$ means stopping gradient.
\end{lemma}

To establish the convergence guarantee of ERC, we use an auxiliary lemma from stochastic approximation~\citep{convergence} and the results in \citet{sac}. \cref{theorem: convergence guarantee} states the convergence guarantee formally.

\begin{theorem}[Convergence of 1-Eigensubspace Regularized Value Approximation]\label{theorem: convergence guarantee}
    Consider the Bellman backup operator $\mathcal{B}$ and a mapping $Q: \mathcal{S} \times \mathcal{A} \rightarrow \mathbb{R}$, and $Q^{k}$ is updated with \cref{eq: ERC update rule tabular setting}. Then the sequence $\{Q^k\}_{k=0}^{\infty}$ will converge to 1-eigensubspace regularized optimal Q value of $\pi$ as $k \rightarrow \infty$. 
\end{theorem}
%

\subsection{Practical Algorithm}

The proposed ERC algorithm, which utilizes the structural information in MDP to improve value approximation, can be formulated as a practical algorithm. We combine the ERC with Soft Actor Critic (SAC) algorithm \cite{sac}. 
For the value approximation, ERC optimizes
\begin{equation}
\begin{aligned}
	J^Q_{\text{ERC}}(\theta) = & \mathbb{E}_{(s,a) \sim \mathcal{D}} \Big[ \frac{1}{2} \Big(  Q(s,a;\theta) - \big(r(s,a) + \gamma \mathbb{E}_{s'\sim P}[ V(s'; \theta')] \big) \Big)^2\Big] + \beta \mathcal{R}_{\text{push}}, 
\end{aligned}
\label{eq: sac-ERC Q objective}
\end{equation}
where $V(s;\theta') = \mathbb{E}_{a\sim \pi(\cdot | s;\phi)} [  Q(s,a;\theta') - \alpha \log \pi(a|s;\phi)]$. For policy improvement, ERC optimizes
\begin{equation}
J^\pi_{\text{ERC}}(\phi) = \mathbb{E}_{s \sim \mathcal{D}} [\mathbb{E}_{a\sim \pi(\cdot \mid s, \phi)}[\alpha \log (\pi(a \mid s;\phi)) - Q(s,a; \theta)]].
\label{eq: sac-ERC policy objective}
\end{equation}

Besides, we also use the automated entropy trick. The temperature $\alpha$ is learned by minimizing 
\begin{equation}
J^{\alpha}_{\text{ERC}}(\alpha)  =  \mathbb{E}_{a\sim \pi^*} [- \alpha \log \pi^* (a|s; \alpha, \phi) -\alpha \mathcal{H}],
\label{eq: sac-ERC alpha objective}
\end{equation}
where $\mathcal{H}$ is a pre-selected target entropy. To help better stabilize the value approximation of ERC, we design a truncation mechanism for $\mathcal{R}_{\text{push}}$
\begin{equation}
\label{eq: ERC beta}
    \begin{aligned}
 \hat{\mathcal{R}}_{\text{push}} = \max \Big\{  \min \Big\{\beta \mathcal{R}_{\text{push}}, \mathcal{R}_{\text{max}} \Big\},  \mathcal{R}_{\text{min}}\Big\}.
\end{aligned}
\end{equation}
%
The practical algorithm is summarized in \cref{alg:ERC}.


\section{Experiments}
\label{sec: exp}
\begin{figure*}[!ht]
	\centering
\scalebox{1}{
\includegraphics[width=1.0 \textwidth]{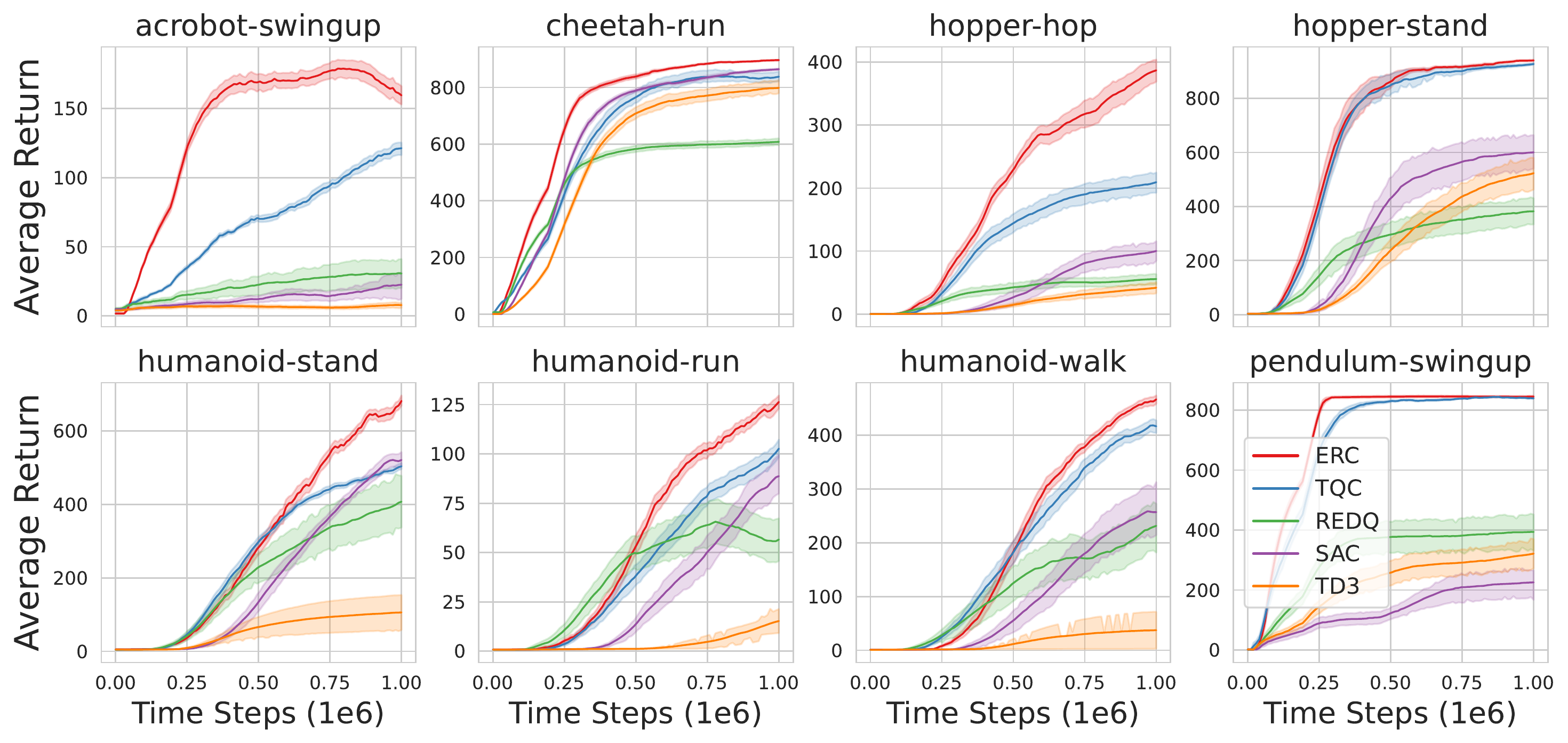}}
\vspace{-0.2in}
\caption{\label{fig: ERC performance}Performance curves for OpenAI gym continuous control tasks on DeepMind Control suite. The proposed algorithm, ERC, is observed to significantly outperform the other tested algorithms. The shaded region represents half of the standard deviation of the average evaluation over 10 seeds. The curves are smoothed with a moving average window of size ten.}
\end{figure*}
 In this section, we thoroughly evaluate the performance of ERC by comparing it to a few baseline methods. Furthermore, we examine the value approximation error and the variance of the value function to gain a deeper understanding of ERC. Additionally, we analyze the individual contributions of each component of ERC to gain insight into its effectiveness.

\subsection{Evaluation Setting}
\begin{table*}[!ht]
    \centering
    \caption{\label{table: ERC performance statistics}Average Return after 1M timesteps of training on DMC. ERC demonstrates state-of-the-art performance on the majority (\textbf{20} out of \textbf{26}) tasks. If not, ERC is still observed to be comparable in performance. Additionally, ERC outperforms its backbone algorithm, SAC, on \textbf{all} tasks by a large margin. Specifically, the ERC algorithm outperforms TQC, REDQ, SAC, and TD3 by 13\%, 25.6\%, 16.6\%, and 27.9\%, respectively. The best score is marked with \colorbox{mine}{colorbox.} $\pm$ corresponds to a standard deviation over ten trials. }
\resizebox{\textwidth}{!}{
\begin{tabular}{ll|lllll}
\toprule
\textbf{Domain} & \textbf{Task} &\textbf{ERC} &\textbf{TQC} &\textbf{REDQ} &\textbf{SAC} &\textbf{TD3}\\
\midrule
Acrobot & Swingup  &  \colorbox{mine} {151.0}{\footnotesize  $\pm$ 36.8} &\;136.3 {\footnotesize $\pm$ 51.9} &\;31.1 {\footnotesize $\pm$ 41.1} &\;26.9 {\footnotesize $\pm$ 47.7} &\;5.3 {\footnotesize $\pm$ 4.7}\\
BallInCup & Catch  & \;979.7 {\footnotesize $\pm$ 1.3} & \colorbox{mine} {981.6}{\footnotesize  $\pm$ 2.3} &\;978.8 {\footnotesize $\pm$ 3.7} &\;980.3 {\footnotesize $\pm$ 3.4} &\;978.9 {\footnotesize $\pm$ 3.6}\\
Cartpole & Balance  &  \colorbox{mine} {998.8}{\footnotesize  $\pm$ 1.2} &\;989.2 {\footnotesize $\pm$ 25.8} &\;984.0 {\footnotesize $\pm$ 6.0} &\;997.7 {\footnotesize $\pm$ 1.5} &\;997.9 {\footnotesize $\pm$ 2.1}\\
Cartpole & BalanceSparse  & \;998.6 {\footnotesize $\pm$ 4.5} &\;899.9 {\footnotesize $\pm$ 268.6} &\;872.1 {\footnotesize $\pm$ 262.7} &\;997.6 {\footnotesize $\pm$ 5.7} & \colorbox{mine} {1000.0}{\footnotesize  $\pm$ 0.0}\\
Cartpole & Swingup  & \;867.8 {\footnotesize $\pm$ 4.7} & \colorbox{mine} {874.3}{\footnotesize  $\pm$ 5.8} &\;828.1 {\footnotesize $\pm$ 17.2} &\;865.1 {\footnotesize $\pm$ 1.6} &\;867.2 {\footnotesize $\pm$ 7.5}\\
Cartpole & SwingupSparse  & \;544.5 {\footnotesize $\pm$ 356.9} & \colorbox{mine} {797.6}{\footnotesize  $\pm$ 32.1} &\;385.5 {\footnotesize $\pm$ 374.7} &\;234.4 {\footnotesize $\pm$ 358.4} &\;157.5 {\footnotesize $\pm$ 314.9}\\
Cheetah & Run  &  \colorbox{mine} {903.3}{\footnotesize  $\pm$ 5.9} &\;853.8 {\footnotesize $\pm$ 80.0} &\;614.2 {\footnotesize $\pm$ 58.2} &\;873.4 {\footnotesize $\pm$ 21.5} &\;811.3 {\footnotesize $\pm$ 102.2}\\
Finger & Spin  &  \colorbox{mine} {988.1}{\footnotesize  $\pm$ 0.6} &\;982.0 {\footnotesize $\pm$ 9.1} &\;940.1 {\footnotesize $\pm$ 33.5} &\;966.3 {\footnotesize $\pm$ 27.1} &\;947.6 {\footnotesize $\pm$ 52.1}\\
Finger & TurnEasy  &  \colorbox{mine} {981.1}{\footnotesize  $\pm$ 5.4} &\;247.4 {\footnotesize $\pm$ 133.6} &\;962.6 {\footnotesize $\pm$ 34.3} &\;920.0 {\footnotesize $\pm$ 91.8} &\;856.5 {\footnotesize $\pm$ 109.3}\\
Finger & TurnHard  &  \colorbox{mine} {964.8}{\footnotesize  $\pm$ 27.5} &\;299.2 {\footnotesize $\pm$ 266.6} &\;927.3 {\footnotesize $\pm$ 99.8} &\;874.1 {\footnotesize $\pm$ 100.1} &\;690.2 {\footnotesize $\pm$ 167.6}\\
Fish & Upright  &  \colorbox{mine} {936.0}{\footnotesize  $\pm$ 12.1} &\;917.1 {\footnotesize $\pm$ 25.6} &\;799.6 {\footnotesize $\pm$ 113.8} &\;898.6 {\footnotesize $\pm$ 50.4} &\;873.6 {\footnotesize $\pm$ 66.7}\\
Fish & Swim  & \;496.8 {\footnotesize $\pm$ 61.6} & \colorbox{mine} {526.6}{\footnotesize  $\pm$ 113.5} &\;159.3 {\footnotesize $\pm$ 100.1} &\;342.4 {\footnotesize $\pm$ 134.5} &\;251.3 {\footnotesize $\pm$ 107.7}\\
Hopper & Stand  &  \colorbox{mine} {943.9}{\footnotesize  $\pm$ 8.9} &\;941.6 {\footnotesize $\pm$ 11.4} &\;393.5 {\footnotesize $\pm$ 225.8} &\;597.8 {\footnotesize $\pm$ 308.8} &\;538.7 {\footnotesize $\pm$ 256.2}\\
Hopper & Hop  &  \colorbox{mine} {405.0}{\footnotesize  $\pm$ 91.1} &\;221.8 {\footnotesize $\pm$ 68.8} &\;56.8 {\footnotesize $\pm$ 36.2} &\;117.4 {\footnotesize $\pm$ 82.2} &\;47.8 {\footnotesize $\pm$ 46.2}\\
Humanoid & Stand  &  \colorbox{mine} {804.5}{\footnotesize  $\pm$ 39.1} &\;494.9 {\footnotesize $\pm$ 145.5} &\;407.2 {\footnotesize $\pm$ 336.8} &\;549.6 {\footnotesize $\pm$ 201.0} &\;110.6 {\footnotesize $\pm$ 206.8}\\
Humanoid & Walk  &  \colorbox{mine} {507.0}{\footnotesize  $\pm$ 37.0} &\;376.4 {\footnotesize $\pm$ 182.5} &\;245.0 {\footnotesize $\pm$ 222.6} &\;248.4 {\footnotesize $\pm$ 220.8} &\;39.3 {\footnotesize $\pm$ 101.9}\\
Humanoid & Run  &  \colorbox{mine} {145.9}{\footnotesize  $\pm$ 10.1} &\;115.6 {\footnotesize $\pm$ 18.6} &\;70.8 {\footnotesize $\pm$ 57.0} &\;83.4 {\footnotesize $\pm$ 56.0} &\;18.1 {\footnotesize $\pm$ 33.9}\\
Pendulum & Swingup  &  \colorbox{mine} {846.6}{\footnotesize  $\pm$ 14.1} &\;834.0 {\footnotesize $\pm$ 30.1} &\;382.6 {\footnotesize $\pm$ 297.0} &\;226.2 {\footnotesize $\pm$ 228.9} &\;338.0 {\footnotesize $\pm$ 232.0}\\
PointMass & Easy  & \;882.3 {\footnotesize $\pm$ 18.2} &\;793.7 {\footnotesize $\pm$ 147.6} &\;880.9 {\footnotesize $\pm$ 16.7} & \colorbox{mine} {889.9}{\footnotesize  $\pm$ 33.1} &\;838.8 {\footnotesize $\pm$ 158.5}\\
Reacher & Easy  &  \colorbox{mine} {986.9}{\footnotesize  $\pm$ 2.3} &\;964.5 {\footnotesize $\pm$ 39.5} &\;970.9 {\footnotesize $\pm$ 24.4} &\;983.5 {\footnotesize $\pm$ 4.2} &\;983.4 {\footnotesize $\pm$ 3.7}\\
Reacher & Hard  &  \colorbox{mine} {981.8}{\footnotesize  $\pm$ 1.7} &\;971.8 {\footnotesize $\pm$ 5.2} &\;964.1 {\footnotesize $\pm$ 24.0} &\;958.6 {\footnotesize $\pm$ 40.9} &\;938.2 {\footnotesize $\pm$ 63.0}\\
Swimmer & Swimmer6  &  \colorbox{mine} {422.0}{\footnotesize  $\pm$ 133.2} &\;356.4 {\footnotesize $\pm$ 107.5} &\;215.8 {\footnotesize $\pm$ 119.0} &\;359.3 {\footnotesize $\pm$ 130.9} &\;289.2 {\footnotesize $\pm$ 133.6}\\
Swimmer & Swimmer15  &  \colorbox{mine} {295.8}{\footnotesize  $\pm$ 113.7} &\;222.8 {\footnotesize $\pm$ 128.6} &\;178.6 {\footnotesize $\pm$ 116.6} &\;264.6 {\footnotesize $\pm$ 136.9} &\;236.7 {\footnotesize $\pm$ 150.1}\\
Walker & Stand  &  \colorbox{mine} {989.6}{\footnotesize  $\pm$ 1.7} &\;986.3 {\footnotesize $\pm$ 4.5} &\;974.0 {\footnotesize $\pm$ 12.6} &\;986.8 {\footnotesize $\pm$ 2.7} &\;983.4 {\footnotesize $\pm$ 4.2}\\
Walker & Walk  &  \colorbox{mine} {974.9}{\footnotesize  $\pm$ 1.6} &\;971.9 {\footnotesize $\pm$ 4.8} &\;957.3 {\footnotesize $\pm$ 10.6} &\;973.5 {\footnotesize $\pm$ 4.4} &\;966.3 {\footnotesize $\pm$ 10.4}\\
Walker & Run  &  \colorbox{mine} {805.1}{\footnotesize  $\pm$ 19.4} &\;770.5 {\footnotesize $\pm$ 31.0} &\;590.9 {\footnotesize $\pm$ 51.6} &\;773.0 {\footnotesize $\pm$ 32.9} &\;712.1 {\footnotesize $\pm$ 65.6}\\
\midrule
Average & Scores & \colorbox{mine} {761.6 } &\;674.1 &\;606.6 &\;653.4 &\;595.3\\
\bottomrule
\end{tabular}}
\vspace{-0.2in}
\end{table*}
\textbf{Baselines.}
We conduct a comparative study of the proposed ERC algorithm with several well-established baselines in the literature. Specifically, we select TD3 \citep{td3} and SAC as our primary baselines as they are commonly used and perform well in various tasks. Additionally, we compare our method with REDQ~\citep{redq} and TQC~\citep{tqc}, which employ different techniques to improve value approximation and reduce the variance of the value function. To ensure a fair comparison, we use the authors' implementation of \href{https://github.com/sfujim/TD3}{TD3} and \href{https://github.com/watchernyu/REDQ}{REDQ} available on Github, and the public implementation of SAC provided in PyTorch~\citep{pytorch_sac}, which is also the basis of our ERC implementation. Besides, we use the implementation of TQC available in the \href{https://sb3-contrib.readthedocs.io/en/master/modules/tqc.html}{stable baselines 3} library and use the default hyper-parameters as suggested by the authors.

\textbf{Environments.}
The experimental suite is the state-based DMControl suite~\citep{dm-control}, which is for physics-based simulation, utilizing the MuJoCo physics engine~\citep{mujoco}. We chose the DMControl suite as it offers a diverse range of environments to benchmark the capabilities of RL algorithms. We facilitate the interactions between the algorithm and environment using Gym~\citep{gym}. We evaluate each algorithm over one million timesteps and obtain the average return of the algorithm every 10k timesteps over ten episodes.

\textbf{Setup.} The magnitude of the regularization effectiveness of ERC is controlled by a hyper-parameter, $\beta$, which is $5e-3$. Additionally, $\mathcal{R}_{\text{max}}$ and $\mathcal{R}_{\text{min}}$ are $1e-2$ and $0$, respectively, for all experiments. The remaining hyper-parameters are consistent with the suggestions provided by \citet{sac}. To ensure the validity and reproducibility of the experiments, unless otherwise specified, we evaluate each tested algorithm over ten fixed random seeds. For more implementation details, please refer to the appendix.

\subsection{Performance Evaluation}
\Cref{fig: ERC performance} shows the learning curves from scratch. \cref{table: ERC performance statistics} depicts the average performance after 1M timesteps training. The results demonstrate the following points: i) ERC outperforms the other tested algorithms in a majority (\textbf{20} out of \textbf{26}) of the environments. Specifically, it outperforms TQC, REDQ, SAC, and TD3 by 13\%, 25.6\%, 16.6\%, and 27.9\%, respectively; ii) ERC substantially improves upon its skeleton algorithm, SAC, by the addition of the $\mathcal{R}_{\text{push}}$ regularization term; iii) Additionally, although ERC does not employ complex techniques to eliminate estimation bias, it still substantially outperforms TQC and REDQ in most environments. Note that both REDQ and TQC leverage an ensemble mechanism to obtain a more accurate and unbiased Q estimation. ERC does not leverage any ensemble critic, distributional value functions, or high UTD ratio, yet it still outperforms them. The results above highlight the potential of utilizing structural information of MDP to improve DRL.

\subsection{Variance and Approximation Error}
\label{sec: value function estimation}

 We select four distinct environments `acrobot-swing', `humanoid-stand', `finger-turn\_easy', and `fish-swim'. This selection is made because ERC has been demonstrated to perform well in the first two environments, while its performance is not as strong in the latter two. Thus, by evaluating ERC across these four environments, a comprehensive assessment of ERC can be obtained.

 \begin{figure*}[ht]
\centering
\vspace{-0.2in}
\begin{minipage}[t]{1\textwidth}
\centering
\includegraphics[width=1.0 \textwidth]{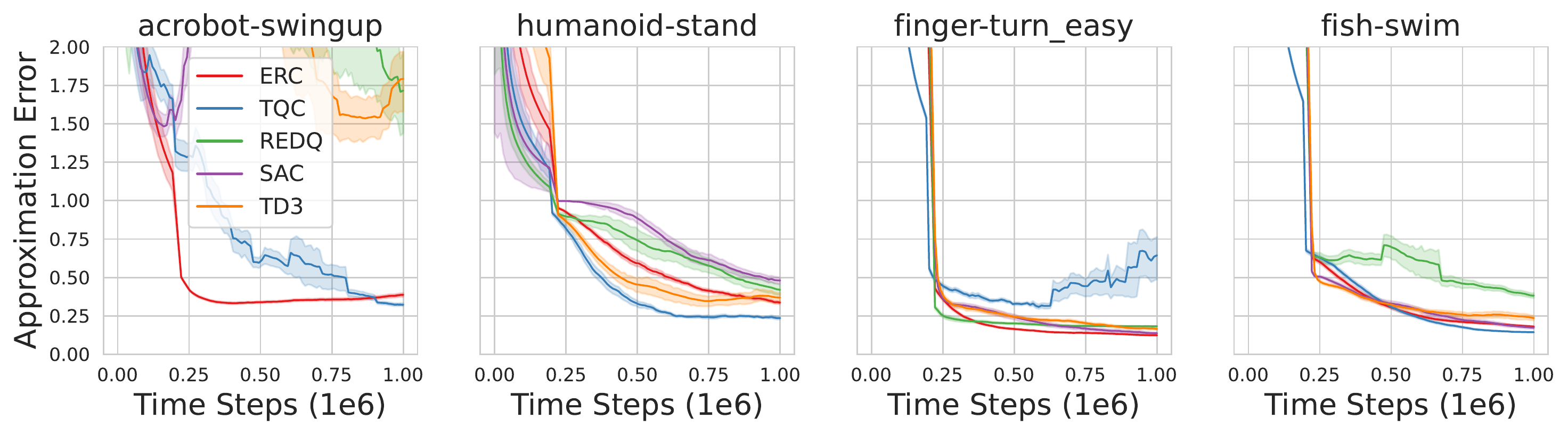}
\caption{\label{fig: bias 4}Approximation error curves.  The results demonstrate that the approximation error of ERC is empirically minimal when compared to other algorithms (such as TQC and REDQ) that are specifically designed to obtain accurate unbiased value estimation. }
\end{minipage}
\begin{minipage}[t]{1\textwidth}
\centering
\includegraphics[width=1.0 \textwidth]{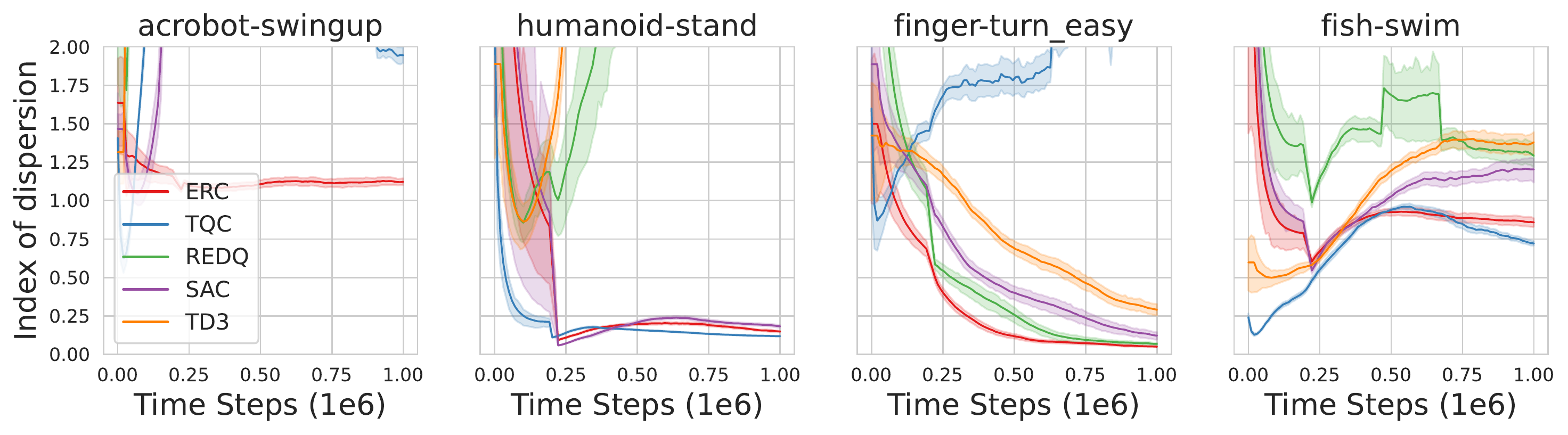}
\caption{\label{fig: dispersion 4}Index of dispersion curves. The results demonstrate that ERC effectively controls the variance. Furthermore, the variance of ERC is observed to be minimal on the selected tasks, even in situations where the performance of ERC is not as good as TQC and REDQ. \vspace{-0.1in}
}
\end{minipage}
\end{figure*}

\begin{table}[!htbp]
    \centering
    \caption{\label{table: ERC ablation}Average Return after 1M timesteps training on DMControl suite. $\beta=5e-3$-trunc means the truncation mechanism is used in the evaluation.
    }
\resizebox{\textwidth}{!}{
\begin{tabular}{lllllllll}
\toprule
\textbf{1M Steps Scores} & \;\textbf{$\beta=$1e-4} &\; \textbf{$\beta=$5e-4} &\;\textbf{$\beta=$1e-3} &\;\textbf{$\beta=$5e-3} &\;\textbf{$\beta=$5e-3-trunc} &\;\textbf{$\beta=$1e-2} &\;\textbf{$\beta=$5e-2} &\;\textbf{SAC}\\
\midrule
{\small Acrobot, Swingup }  & \;151.2 {\footnotesize $\pm$ 70.0} & \colorbox{mine} {152.7}{\footnotesize  $\pm$ 43.4} &\;91.4 {\footnotesize $\pm$ 54.3} &\;151.9 {\footnotesize $\pm$ 69.3} &\;151.0 {\footnotesize $\pm$ 36.8} &\;129.9 {\footnotesize $\pm$ 58.7} &\;114.7 {\footnotesize $\pm$ 31.4} &\;26.9 {\footnotesize $\pm$ 47.7}\\
{\small Humanoid, Stand }  & \;418.2 {\footnotesize $\pm$ 255.3} &\;755.8 {\footnotesize $\pm$ 88.3} &\;692.1 {\footnotesize $\pm$ 200.2} &\;699.0 {\footnotesize $\pm$ 129.3} & \colorbox{mine} {804.5}{\footnotesize  $\pm$ 39.1} &\;742.4 {\footnotesize $\pm$ 130.8} &\;550.0 {\footnotesize $\pm$ 217.2} &\;549.6 {\footnotesize $\pm$ 201.0}\\
{\small Finger, TurnEasy }  & \;940.4 {\footnotesize $\pm$ 49.5} &\;919.8 {\footnotesize $\pm$ 80.5} &\;959.2 {\footnotesize $\pm$ 37.3} &\;979.4 {\footnotesize $\pm$ 5.5} & \colorbox{mine} {981.1}{\footnotesize  $\pm$ 5.4} &\;979.9 {\footnotesize $\pm$ 5.3} &\;959.7 {\footnotesize $\pm$ 43.7} &\;920.0 {\footnotesize $\pm$ 91.8}\\
{\small Fish, Swim }  & \;448.5 {\footnotesize $\pm$ 102.8} &\;436.9 {\footnotesize $\pm$ 71.4} &\;468.3 {\footnotesize $\pm$ 130.5} &\;390.0 {\footnotesize $\pm$ 65.8} & \colorbox{mine} {496.8}{\footnotesize  $\pm$ 61.6} &\;414.4 {\footnotesize $\pm$ 55.0} &\;434.9 {\footnotesize $\pm$ 61.4} &\;342.4 {\footnotesize $\pm$ 134.5}\\
\bottomrule
\end{tabular}
}
\vspace{-0.1in}
\end{table}
\textbf{Approximation Error.} ERC aims to push the Bellman error to 1-eigensubspace to obtain an efficient and stable value approximation as discussed in \cref{sec: Improve value approximation via Eigensubspace regularization}. Thus we study the approximation error of ERC on the DMControl suite. We obtain the true value $Q^*$ using the Monte Carlo method and compare it to the estimated Q. We normalize the approximation error by the estimation value to eliminate the differences in scale, and the absolute value of the approximation error is used in each sample to eliminate inaccuracies caused by offsetting positive and negative errors. The results, presented in \cref{fig: bias 4}, indicate that ERC exhibits minimal approximation error compared to other methods that obtain accurate and unbiased value estimates.

\textbf{Variance Reduction.} Above analysis shows that ERC reduces the value function variance. To investigate this claim in \cref{sec: Improve value approximation via Eigensubspace regularization}, we empirically analyze the variance of ERC. To eliminate the effect of the size of the value function on the variance, we use \textit{index of dispersion} \citep{indices}, which can be considered as a variance normalized by its mean, to evaluate the variance of algorithms. Our results, presented in \cref{fig: dispersion 4}, indicate that i) the variance of the value function of ERC is the lowest or equally low as other methods (TQC, and REDQ) specifically designed for variance reduction. And ii) ERC effectively controls variance even in environments where its performance is not as strong as other methods such as TQC and REDQ. These findings provide strong evidence for the theoretical claim that ERC reduces variance of the value function.

\subsection{Ablation}

The component introduced by ERC, as outlined in \cref{eq: ERC loss network form}, is sole ${\mathcal{R}}_{\text{push}}$, in which $\beta$ regulates the efficiency of pushing the Bellman error towards the 1-eigensubspace. Therefore, investigating the selection of $\beta$ can aid in comprehending the impact of hyper-parameter on ERC.
We vary $\beta$ and eliminate the truncation mechanism outlined in \cref{eq: ERC beta}. Experiments are conducted on the same four environments as described in \cref{sec: value function estimation}. The results, presented in \cref{table: ERC ablation}, indicate the following: i) the value of $\beta$ influences the empirical performance of the ERC, yet ERC with various $\beta$ values consistently outperforms its skeleton algorithm, SAC, in most settings. That demonstrates the effectiveness of our proposed value approximation method; ii) Truncating parameters can improve the performance of ERC on specific tasks (Humanoid-Stand, Fish-Swim); iii) Carefully tuning both the hyper-parameter $\beta$ and the truncation mechanism ensures optimal performance.

\section{Related Work}

\subsection{Value Function Approximation}
\citet{nature_dqn} utilizes neural networks (NN) to approximate value function by TD learning~\citep{q-learning, td-gammon}. Combined with NN, several derivative value approximation methods~\citep{nature_dqn,c51,qr-dqn,iqn} exist. Some ~\citep{c51,iqn,qr-dqn,tqc} use a distributional view to obtain a better value function. Some reduce bias and variance of value function approximation~\cite{hasselt2010double,ddqn,td3, lan2019maxmin, redq, wd3}. Others~\citep{retrace,impala,sac} develop more robust value function approximation methods. ERC utilizes structural information from MDP to improve value approximation, which distinguishes it from previous work.

\subsection{Using Structural Information of MDPs}
Recent work~\citep{low-rank-1, lowrank2, low-rank-3, low-rank-4} focus on the property of low-rank MDP. They study learning an optimal policy assuming a low-rank MDP given the optimal representation. Nevertheless, such a setting is not practical as the optimal representation is computationally intractable. \Citet{tongzhen22sdrr} proposed a practical algorithm under low-rank MDP assumption, which, different from ERC, involves estimating the dynamics of MDP. Some other work \citep{lyle2021understanding, lyle2021effect,  agarwal2019reinforcement} discuss the architecture of MDP from a matrix decomposition perspective. \Citet{lyle2021effect} establishes a connection between the spectral decomposition of the transition operator and the representations of $V$ function induced by a variety of auxiliary tasks. ERC differs from \citet{lyle2021effect} in the following aspects: i) We analyze the dynamics of the Q approximation error, which is more commonly studied in DRL literature; ii) We consider the learning process of $Q$ function, whereas \citet{lyle2021effect} considers the dynamics of representations.

\section{Conclusion}
In this work, we examine the eigensubspace of the TD dynamics and its potential use in improving value approximation in DRL. We begin by analyzing the matrix form of the Bellman equation and subsequently derive the dynamics of the approximation error through the solution of a differential equation. This solution depends on the transition kernel of the MDP, which motivates us to perform eigenvalue decomposition, resulting in the inherent path of value approximation in TD. To the best of our knowledge, this inherent path has not been leveraged to design DRL algorithms in previous work. Our insight is to improve value approximation by directing the approximation error towards the 1-eigensubspace, resulting in a more efficient and stable path. Thus, we propose the ERC algorithm with a theoretical convergence guarantee. Theoretical analysis and experiments demonstrate ERC results in a variance reduction, which validates our insight. Extensive experiments on the DMControl suite demonstrate that ERC outperforms state-of-the-art algorithms in the majority of tasks. The limitation is that ERC is evaluated on the DMControl suite. Verifying the effectiveness of ERC on other suites is left for future work. Our contributions represent a significant step forward in leveraging the rich inherent structure of MDP to improve value approximation and ultimately enhance performance in RL.

\textbf{Acknowledgements.} Thanks to Xihuai Wang, Yucheng Yang, and Qifu Hu for their helpful discussions. This research was supported by Grant 01IS20051 and Grant 16KISK035 from the German Federal Ministry of Education and Research (BMBF).

\section*{Ethical Statement}
This work investigates the eigensubspace of the Temporal-Difference algorithm of reinforcement learning and how it improves DRL. The experimental part only uses a physical simulation engine DMControl. Therefore this work does not involve any personal information. This work may be used in areas such as autonomous vehicles, games, and robot control.
\normalem
\bibliographystyle{splncs04nat}
\bibliography{sample.bib}


\appendix
\onecolumn
\section{Notations}

\begin{table*}[!ht]
    \centering
    \caption{\label{app table: notations}Notations used in this work. 
    }
\begin{tabular}{l|l}
\toprule
\textbf{Symbol} & \textbf{Description} \\
\midrule
$\mathcal{S}$ & State space  \\
$\mathcal{A}$ & Action space \\
$R$ & Reward function \\
$R_t$ & Return $R_t=\sum_{i=t}^{T}\gamma^{i-t} r(s_i, a_i)$ \\
$r$ & A reward \\
$P$ & Transition kernel (or transition probability), $P: \mathcal{S} \times \mathcal{A} \rightarrow p(s)$ \\
$\gamma$ & Discount factor, $\gamma \in [0, 1)$ \\
$\mathcal{D}$ & Replay buffer \\
$\rho_0$ & Initial state distribution \\
$p(s)$ & A state distribution \\
$Q^{\pi}(s, a)$ & Action value function \\
$\tau$ & Trajectory, a sequence of state and action \\
$\pi$ & Policy \\
$V^\pi(s)$ & State value function given policy $\pi$ \\
$\alpha$ & Temperature for MaxEnt RL~\citep{sac} \\
$\mathcal{H}$ & Pre-selected target entropy. \\
$\theta$ & neural network parameters for value function \\
$\phi$ & neural network parameters for policy function \\
\midrule
$P^{\pi}_{s,a,s',a'}$ & Transition matrix, $P^{\pi}_{s,a,s',a'}:= P(s'|s,a)\pi (a'|s').$ \\
$\mathbf{V}^\pi$ & Vector of all V value with length $|\mathcal{S}|$ \\
$\mathbf{Q}^\pi$ & Vector of all Q value with length $|\mathcal{S}|\cdot |\mathcal{A}|$ \\
$\mathbf{r}$ & Vector of all reward with length $|\mathcal{S}|\cdot |\mathcal{A}|$  \\
$\mathbf{e}$ & A column of all ones \\ 
$I$ & Identity matrix \\ 
$(\lambda_i, H_i)$ & Eigenpair \\
$(\mathfrak{B}, \| \cdot \|)$ & A Banach space equipped with a norm  $\| \cdot \|$ \\
$\beta$ & Hyper-parameter for ERC, controlling trend to 1-eigensubspace \\
\bottomrule
\end{tabular}
\end{table*}

\section{Theoretical Derivations}
In this section, theorems and lemmas in the main text are restated, and are given the related proof.

 




\subsection{An Inherent Path of Value Approximation}\label{app: sec: dynamics of bellman error}






\begin{lemma}[Dynamics of approximation error] \label{app thm: dynamics of Q}
    Consider a continuous sequence $\{Q_t | t \geq 0 \}$, satisfy \cref{eq: dynamics differential form} with initial condition $\mathbf{Q}_0$ at time step $t=0$, then
    \begin{equation}\label{app eq: dynamics of approximation error}
        \mathbf{Q}_t - \mathbf{Q}^* = \exp \{ -t (I - \gamma P^\pi)   \} (\mathbf{Q}_0 - \mathbf{Q}^*) .
    \end{equation}
\end{lemma}

\begin{proof}
    This ordinary differential equation can be solved directly with the help of \cref{corollary: Q pi}.
    \hfill$\square$
\end{proof}


\begin{theorem}
\Cref{assumption: P pi diagnoalize} holds,
then $\mathbf{Q}_t - \mathbf{Q}^* = \alpha_1 \exp\{ t(\gamma \lambda_1 -1 \}H_1 + \sum_{i=2}^{|\mathcal{S}|\cdot |\mathcal{A}|} \alpha_i \exp \{ t(\gamma \lambda_i -1) \} H_i = \alpha_1 \exp\{ t(\gamma \lambda_1 -1 \}H_1 + o\Big( \alpha_1 \exp\{ t(\gamma \lambda_1 -1 \} \Big)$ 
\end{theorem}

\begin{proof}
We have
    \begin{equation}
        \mathbf{Q}_0 - \mathbf{Q}^* = \sum_i \alpha_i H_i,
    \end{equation}
where $\alpha_i$ is some constants. Note that $\mathbf{Q}_t$ is in the column space of $P^\pi$ and $\{H_i\}$ can is a basis of $P^\pi$.

Recall \cref{thm: dynamics of Q}, we have 
    \begin{equation}    
    \begin{aligned}    
\mathbf{Q}_t - \mathbf{Q}^* &= \exp \{ -t (I - \gamma P^\pi)   \} (\mathbf{Q}_0 - \mathbf{Q}^*) \\    &= \exp \{ -t (I - \gamma P^\pi)   \}   \sum_i \alpha_i H_i\\    
&=      \sum_i \alpha_i \exp \{ -t (I - \gamma P^\pi)   \} H_i          \\
&= \sum_i \alpha_i \exp \{ -t (1 - \gamma \lambda_i)   \} H_i  \\
&= \alpha_1 \exp\{ t(\gamma \lambda_1 -1 ) \} H_1 + \sum_{i=2} \alpha_i \exp\{ t(\gamma \lambda_i -1 )\} H_i \\    
&= \alpha_1 \exp\{ t(\gamma -1)\}H_1 + o\Big( \alpha_1 \exp\{ t(\gamma  -1 \} \Big).      
\end{aligned}    
\end{equation}
The fourth equation holds because $\exp\{ -t(I-\gamma P^\pi )\}$ is also diagonalizable under the same basis $\{H_i\}$, with the eigenvalue $\exp \{ t(\gamma \lambda_i -1 )\}$, $i \in \{1,2, \cdots, |\mathcal{S}|\cdot |\mathcal{A}| \}$.
\hfill$\square$
\end{proof}

\begin{lemma}
    Consider a Banach space $(\mathfrak{B}, \| \cdot \|)$ of dimension N, and let the N-dimensional Bellman error at timestep $t$, represented by $\mathbf{B}^t$, have coordinates $(B_1, B_2, \cdots, B_N)$ in $(\mathfrak{B}, \| \cdot \|)$. Within this Banach space, the projected point in the 1-eigensubspace, which is closest to $B^t$, is $Z^t$ whose coordinates are $(z^t, z^t,\cdots, z^t)$ where $z^t = \frac{1}{N}\sum_{j=1}^{N} B_i^t$.
\end{lemma}

\begin{proof}
    Assume $Z$ is in 1-eigensubspace whose coordinates are $(z^t, z^t,\cdots, z^t)$ with dimension N. The distance from $Z^t$ to $N^t$ is $\sum_i^N (B_i^t - z^t)^2$. To find the coordinates of $Z^t$, we can build the following optimization problem:
    \begin{equation}
        \min_z \sum_i^N (B_i^t - z^t)^2,
    \end{equation}
    which is a convex optimization problem. The solution is $z^t = \frac{1}{N}\sum_{j=1}^{N} B_i^t$.

    \hfill$\square$
\end{proof}

\subsection{Theoretical Analysis}

\begin{lemma}\label{app theorem: ERC update rule}
Given the ERC update rules in \cref{eq: ERC loss network form}, ERC updates the value function in tabular form in the following way
\begin{equation}\label{app eq: ERC update rule tabular setting}
\begin{aligned}
     Q_{t+1} = (1- \alpha_t (1+\beta) )  Q_t + \alpha_t (1+\beta) \mathcal{B}Q_t -  \alpha_t\beta C_t,
\end{aligned}
\end{equation}
where $\mathcal{B}Q_t(s_t,a_t) = r_t + \gamma \mathbb{E}_{s_{t+1},a_{t+1}} Q_t(s_{t+1}, a_{t+1})$, $C_t = 2 \mathbb{NG}(\mathbb{E} [ \mathcal{B}Q_t - Q_t ])$, and $\mathbb{NG}$ means stopping gradient.
\end{lemma}

\begin{proof}
For $\mathcal{R}_{\text{push}}$, we have
    \begin{equation}
        \begin{aligned}
                \mathcal{R}_{\text{push}} &= \mathbb{E}_{s,a} \Big( (Q-\mathbb{NG}(\mathcal{B}Q)) - \mathbb{NG}(\mathbb{E} [Q - \mathcal{B}Q] ) \Big) ^2  \\ 
                &= \mathbb{E}_{s,a} ( (Q-\mathbb{NG}(\mathcal{B}Q)) )^2 - 2 \mathbb{E} [\mathbb{NG}(\mathcal{B}Q) - Q ] \mathbb{NG}(\mathbb{E} [\mathcal{B}Q - Q ]) +  (\mathbb{NG}(\mathbb{E}[ Q - \mathcal{B}Q ] ))^2  \\
                &= \mathcal{L}_{\text{PE}} - 2 \mathbb{E} [\mathbb{NG}(\mathcal{B}Q) - Q ] \mathbb{NG}(\mathbb{E} [ \mathcal{B}Q - Q ]) + (\mathbb{NG}(\mathbb{E}[\mathcal{B}Q - Q  ] ))^2 \\
                &= \mathcal{L}_{\text{PE}} - C_t \mathbb{E} [ \mathbb{NG}(\mathcal{B}Q) - Q  ] + \underbrace{ (\mathbb{NG}(\mathbb{E}[ Q - \mathcal{B}Q ] ))^2}_{ \text{No contribution to gradient}},
        \end{aligned}
    \end{equation}
where $C_t = 2 \mathbb{NG}(\mathbb{E} [ \mathcal{B}Q_t - Q_t ])$. 

Now we consider the gradient of ERC w.r.t. $Q$.

\begin{equation}
    \begin{aligned}
        \nabla_Q  \mathcal{L}_{\text{ERC}} &= \nabla_Q \mathcal{L}_{\text{PE}} + \nabla_Q \beta \mathcal{R}_{\text{push}} \\
        &= (1+\beta) \nabla_Q \mathcal{L}_{\text{PE}}  -  \beta C_t \nabla_Q \mathbb{E} [\mathbb{NG}( \mathcal{B}Q) - Q  ]  \\
        &= (1+\beta) (Q - \mathcal{B}Q) + \beta C_t .
    \end{aligned}
\end{equation}
Thus we have the following update rule
\begin{equation}
    \begin{aligned}
          Q_{t+1} & = Q_t - \alpha_t \nabla_{Q_t}  \mathcal{L}_{\text{ERC}} \\
          &= Q_t - \alpha_t (1+\beta) (Q_t - \mathcal{B}Q_t) -  \alpha_t\beta C_t \\
        &= (1- \alpha_t (1+\beta) )  Q_t + \alpha_t (1+\beta) \mathcal{B}Q_t -  \alpha_t\beta C_t.
    \end{aligned}
\end{equation}


\end{proof}

\begin{theorem}[Convergence of 1-Eigensubspace Regularized Value Approximation]\label{app thm: convergence}
    Consider the Bellman backup operator $\mathcal{B}$ and a mapping $Q: \mathcal{S} \times \mathcal{A} \rightarrow \mathbb{R}$, and $Q^{k}$ is updated with \cref{app eq: ERC update rule tabular setting}. Then the sequence $\{Q^k\}_{k=0}^{\infty}$ will converge to 1-eigensubspace regularized optimal Q value of $\pi$ as $k \rightarrow \infty$. 
\end{theorem}

\begin{proof}
According to \cref{app theorem: ERC update rule}, the update rule of ERC can be rewritten as 

\begin{equation}\label{app eq: rewrite erc update rule}
\begin{aligned}
        Q^{k+1} & \leftarrow (1- \alpha_t (1+\beta) )  Q^k + \alpha_t (1+\beta) \big( \mathcal{B}Q^k -  \frac{\beta}{1+\beta} C^k \big) \\
        & \leftarrow (1- \alpha_t (1+\beta) )  Q^k + \alpha_t (1+\beta) \big( r(s,a) + \gamma \mathbb{E}_{s', a'}Q_t(s', a')  -  \frac{\beta}{1+\beta} C^k \big) \\
        & \leftarrow (1- \alpha_t (1+\beta) )  Q^k + \alpha_t (1+\beta) \big( r(s,a) + \frac{\beta}{1+\beta} \mathbb{E}_{s,a}[Q^k -\mathcal{B} Q^k] \\
        & \quad +  \gamma \mathbb{E}_{s', a'}Q^k(s', a') \big)
\end{aligned}
\end{equation}
From \cref{app eq: rewrite erc update rule}, the update target in ERC is $ r(s,a) + \gamma \mathbb{E}_{s', a'}Q_t(s', a')  -  \frac{\beta}{1+\beta} C^k $. 
    Define the 1-eigensubspace regularized reward as $r_{erc}(s_t,a_t) \coloneqq r(s,a) + \frac{\beta}{1+\beta} \mathbb{E}_{s,a}[Q^k -\mathcal{B}^\pi Q^k]$, and the update rule of ERC can be further rewritten as 

\begin{equation}
    Q(s,a) \leftarrow r_{erc}(s,a) + \gamma \mathbb{E}[Q(s', a')],
\end{equation}
    and apply the standard convergence results for policy evaluation \citep{convergence,rl}, which concludes the proof. 
    \hfill$\square$
\end{proof}


\subsection{ERC*}\label{app sec: erc*}
In the text, ERC*, which leverages optimal $Q^*$ to push approximation error to 1-eigensubspace, is compared to ERC and TD in \cref{fig: distance 2 eigensubspace and approximation error}. We give a detailed description of the ERC* in this section. Following \cref{lemma: ERC in tabular form}, the update rule of ERC* can be given as 

\begin{equation}
    Q_{t+1} = (1- \alpha_t (1+\beta) )  Q_t + \alpha_t \mathcal{T}Q_t +  \alpha_t\beta (Q^* + \mathbb{E}(Q_t - Q^*) ).
\end{equation}

The procedure of the proof is the same as in \cref{lemma: ERC in tabular form}.
\section{Additional Details Regarding Experiments}
In this section, we provide a detailed description of the experimental procedures and configurations used to generate the tables and figures presented in the main text. 

\textbf{Implementations.} To ensure the validity and reproducibility of the experiments, we have fixed all random seeds, including but not limited to those used in PyTorch, Numpy, Gym, Random, and CUDA packages, across all experiments. For random seeds, unless otherwise specified, we evaluate each tested algorithm over 10 fixed random seeds. For more implementation details, please refer to our  \href{https://sites.google.com/view/erc-ecml23/}{code}.

\textbf{Comparison with REDQ and TQC.} Both REDQ and TQC leverage an ensemble mechanism to obtain a more accurate and unbiased Q estimation. Specifically, REDQ uses 10 critics and the UTD ratio is 20, i.e., REDQ performs gradient updates twenty times for every interaction with the environment. As for TQC, it uses 5 distributional critics. However, ERC does not leverage any ensemble critic, distributional value functions, or high UTD ratio, yet it still outperforms them.

\textbf{\Cref{fig: distance 2 eigensubspace and approximation error}.} The experiments are conducted on \href{https://github.com/openai/gym/blob/master/gym/envs/toy_text/frozen_lake.py}{FrozenLake-v1 environment}. The true Q values are computed through the Monte Carlo method. The shaded area represents a standard deviation over trials. In the case of no additional description, the shaded areas carried in all subsequent figures indicate a standard deviation over ten random seeds. The environment used in the experiment has 16 states and 4 actions. The learning rate is 0.01. The discounted factor is 0.9. A fixed $\beta$ value of 0.3 is used for both ERC and ERC*.

\textbf{\Cref{fig: ERC performance}.} The performance curves for OpenAI gym continuous control tasks on the DeepMind Control suite. The shaded region represents a 50\% standard deviation of the average evaluation over 10 seeds. and the curves are smoothed with a moving average window of size 10. The evaluation of each algorithm is conducted over a period of 1 million timesteps, with the average return of the algorithm being evaluated every 10k timesteps over ten episodes.

\textbf{\Cref{fig: dispersion 4}.}  The shaded region in these figures also represents a 50\% standard deviation of the average evaluation over 10 seeds and the curves are smoothed with a moving average window of size 10. The evaluation of each algorithm is conducted over a period of 1 million timesteps, with the index of dispersion of the algorithm being evaluated every 10k timesteps over ten episodes. For a better visualization effect, the y-axis is clipped to $[-1,5]$.

\textbf{\Cref{fig: bias 4}.} The experimental settings are consistent with those of \cref{fig: dispersion 4}.

\textbf{\Cref{table: ERC performance statistics}.} Average Return after 1M timesteps training on DMControl suite. $\pm$ indicates a standard deviation over ten trials. The DMControl suite benchmarks contain 28 tasks, and our evaluation of the algorithms was conducted on 26 tasks, with the tasks `acrobot-swingup\_sparse' and `manipulator-bring\_ball' excluded due to the inability of all tested DRL algorithms to produce meaningful results on these challenging tasks.

\textbf{\Cref{table: ERC ablation}.} The experimental settings are consistent with those of \cref{table: ERC performance statistics}.

\section{Additional Experimental Results}

In this section, we present additional experimental results.

\subsection{The Inherent Path in Practice}
\Cref{sec: dynamics of bellman error} discusses the inherent path of approximation error. Does this occur in the practical scene? To empirically show the path of the approximation error, we visualize the path of approximation error given a fixed policy. The results are given in \cref{app fig: dynamics td}, where we perform experiments on \href{https://github.com/openai/gym/blob/master/gym/envs/toy_text/frozen_lake.py}{FrozenLake-v1 environment} since the true value $Q^*$ can be evaluated by the MC Method. For comparison, we also visualize the path of value function learning by the Monte Carlo method. The approximation error of the TD method is optimized toward the 1-eigensubspace before ultimately converging to zero rather than directly toward the optimum that the Monte Carlo (MC) method acts. The inherent path in \cref{app fig: sub dynamics td} is consistent with \cref{theorem: path of td method}. Thus, such a path is non-trivial and motivates us to improve the value approximation by guiding the Bellman error to 1-eigensubspace, resulting in an efficient and robust path. 

\begin{figure}[!htbp]
\centering
\hspace{-0.2in}
\subcaptionbox{\label{app fig: sub dynamics td}Path of TD method}
{\includegraphics[width=0.5\textwidth]{figure/dynamics_td.pdf}}
\hspace{-0.1in}
\subcaptionbox{\label{app fig: sub dynamics mc}Path of MC method}
{\includegraphics[width=0.5\textwidth]{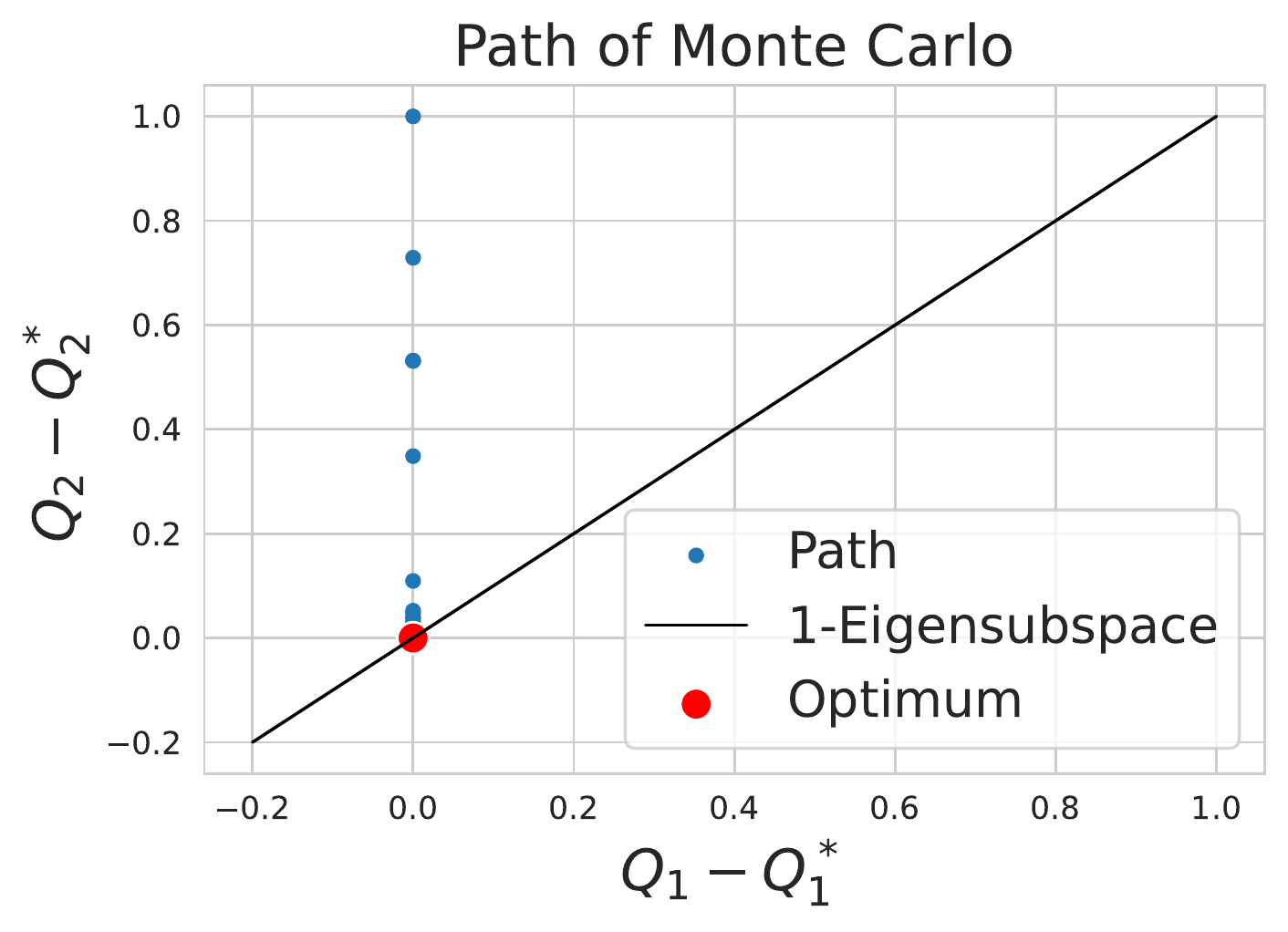}}
\caption{\label{app fig: dynamics td}Value approximation path for TD and MC methods. (a) illustrates that there exists an inherent path that approximation error approaches 1-eigensubspace before converging to zero. The empirical fact is consistent with our theoretical analysis in \cref{theorem: path of td method}. (b) demonstrates the path of the Value approximation error for the MC method. The value function approaches the true value with the shortest path because the MC method obtains the unbiased true value of the objective by a large number of samples.}
\vspace{-15pt}
\end{figure}

\subsection{Value Approximation on Additional Environment}
In the text, we show the value approximation process for various algorithms in \cref{fig: distance 2 eigensubspace and approximation error} on FrozenLake-v1 environment. Now We offer the Value function approximation process for multiple methods on \href{https://www.gymlibrary.dev/environments/toy_text/cliff_walking/}{CliffWalking-v0} environment in \cref{app fig: distance 2 eigensubspace and approximation error}.

\begin{figure}[!t]
\centering 
\subcaptionbox{\label{app fig: sub distance2eigensubspace}Distance to 1-eigensubspace}
{\includegraphics[width=0.5\textwidth]{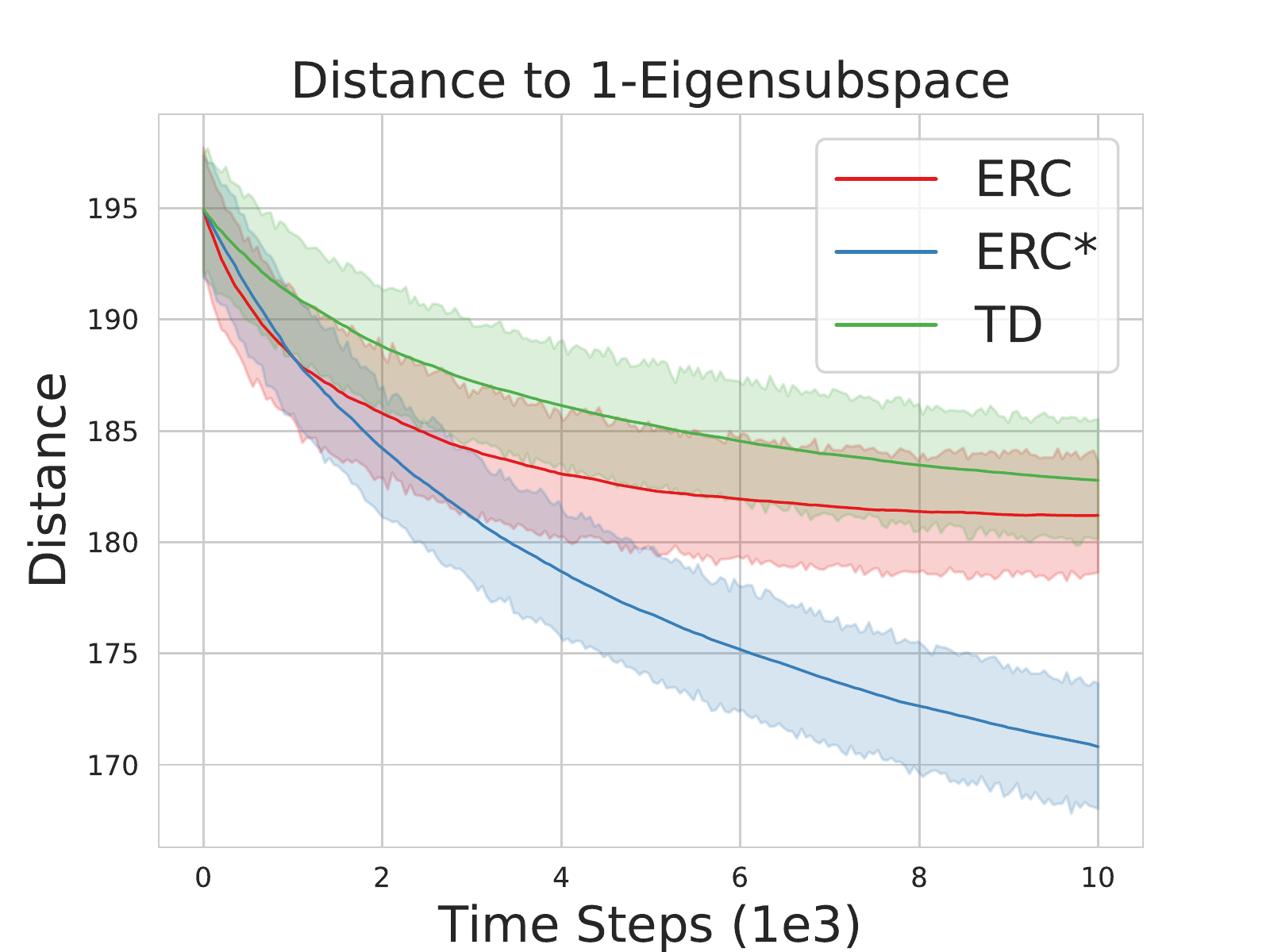}}
\hspace{-0.2in}
\subcaptionbox{\label{app fig: sub Approximation_Error}Approximation error}
{\includegraphics[width=0.5\textwidth]{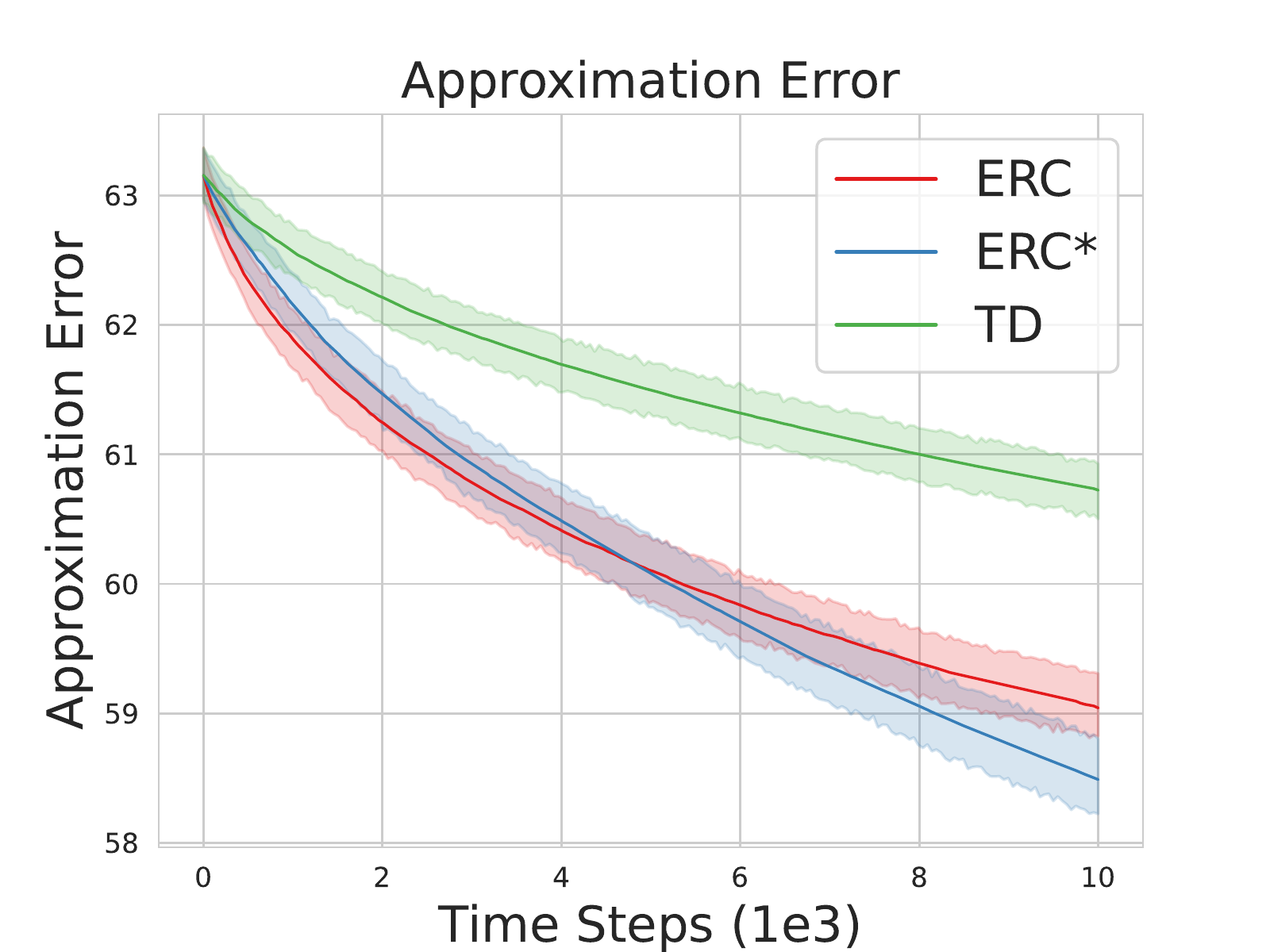}}
\caption{\label{app fig: distance 2 eigensubspace and approximation error}Value function approximation process for various methods on \href{https://www.gymlibrary.dev/environments/toy_text/cliff_walking/}{CliffWalking-v0} environment. (a) illustrates the distance between the approximation error and the 1-eigensubspace for various methods, where ERC* denotes ERC utilizing an oracle $Q*$ to push the approximation error towards the 1-eigensubspace. The results demonstrate that the ERC method is closer to the 1-eigensubspace at the same time compared to the TD method. Besides, the ERC* is closer to 1-eigensubspace because it leverages oracle true $Q^*$. (b) represents the absolute approximation error for various algorithms. The result illustrates that the ERC method has a smaller approximation error at the same time than the TD method. ERC is close to ERC* in this metric, which means that although ERC* is more efficient near the 1-eigensubspace, the performance of the ERC method is also almost the same as ERC* in terms of the approximation error metric that we really care about. The shaded area represents a standard deviation over ten trials.}
\vspace{-15pt}
\end{figure} 

\subsection{Additional Figures}
Due to space limitations, we present only selective figures in the text. In this section, we give the remaining figures for the 26 tasks.

\Cref{app fig: ERC performance 18} presents the learning curves from scratch. \Cref{app fig: bias 22} shows approximation error curves. \Cref{app fig: dispersion 22} is index of dispersion curves. We note that TQC has the best approximation error on `cartpole-swingup\_sparse' task. This is because TQC outperforms the other algorithms substantially on this task as well.

\begin{figure*}[!htp]
	\centering
\scalebox{1}{
\includegraphics[width=1.0 \textwidth]{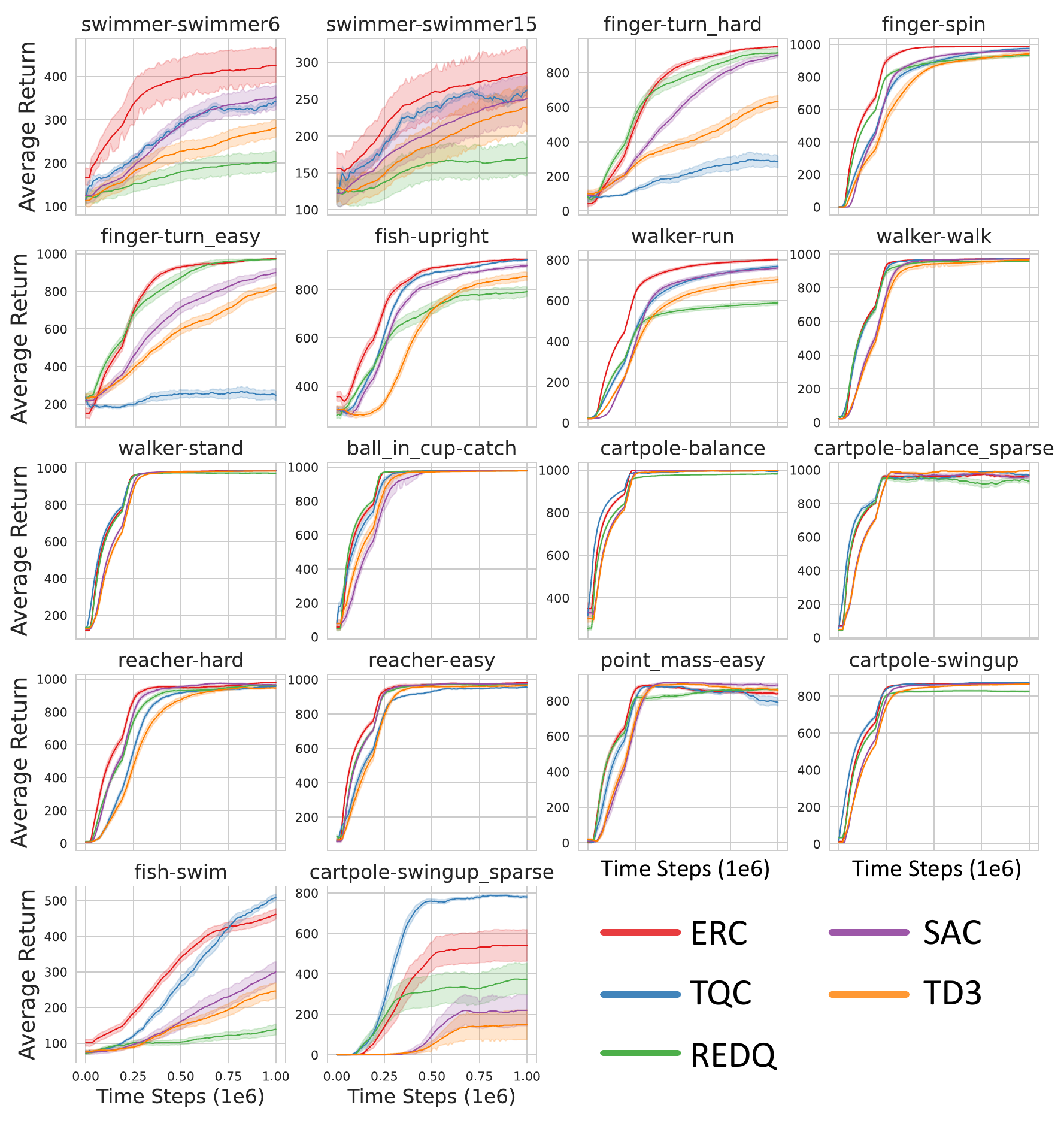}}
\vspace{-0.2in}
	\caption{\label{app fig: ERC performance 18}Performance curves for OpenAI gym continuous control tasks on DeepMind Control suite. The proposed algorithm, ERC, is observed to significantly outperform the other tested algorithms. The shaded region represents half of the standard deviation of the average evaluation over 10 seeds. The curves are smoothed with a moving average window of size ten.}
 \vspace{-0.2in}
\end{figure*}

\begin{figure*}[!htp]
	\centering
\scalebox{1}{
\includegraphics[width=1.0 \textwidth]{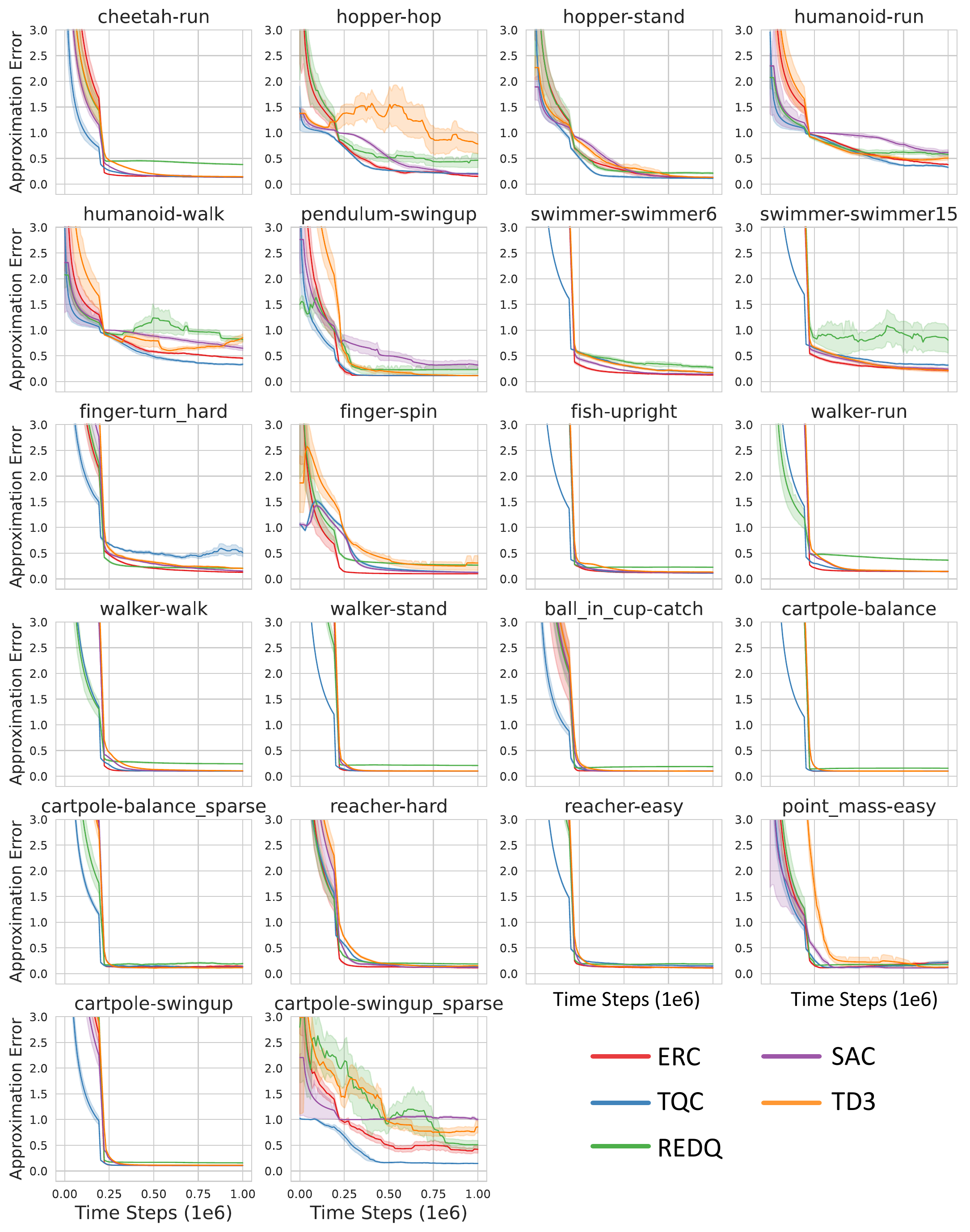}}
\vspace{-0.2in}
\caption{\label{app fig: bias 22}Approximation error curves.  The results demonstrate that the approximation error of ERC is empirically minimal when compared to other algorithms (such as TQC and REDQ) that are specifically designed to obtain accurate unbiased value estimation. }
\end{figure*}

\begin{figure*}[!htp]
	\centering
\scalebox{1}{
\includegraphics[width=1.0 \textwidth]{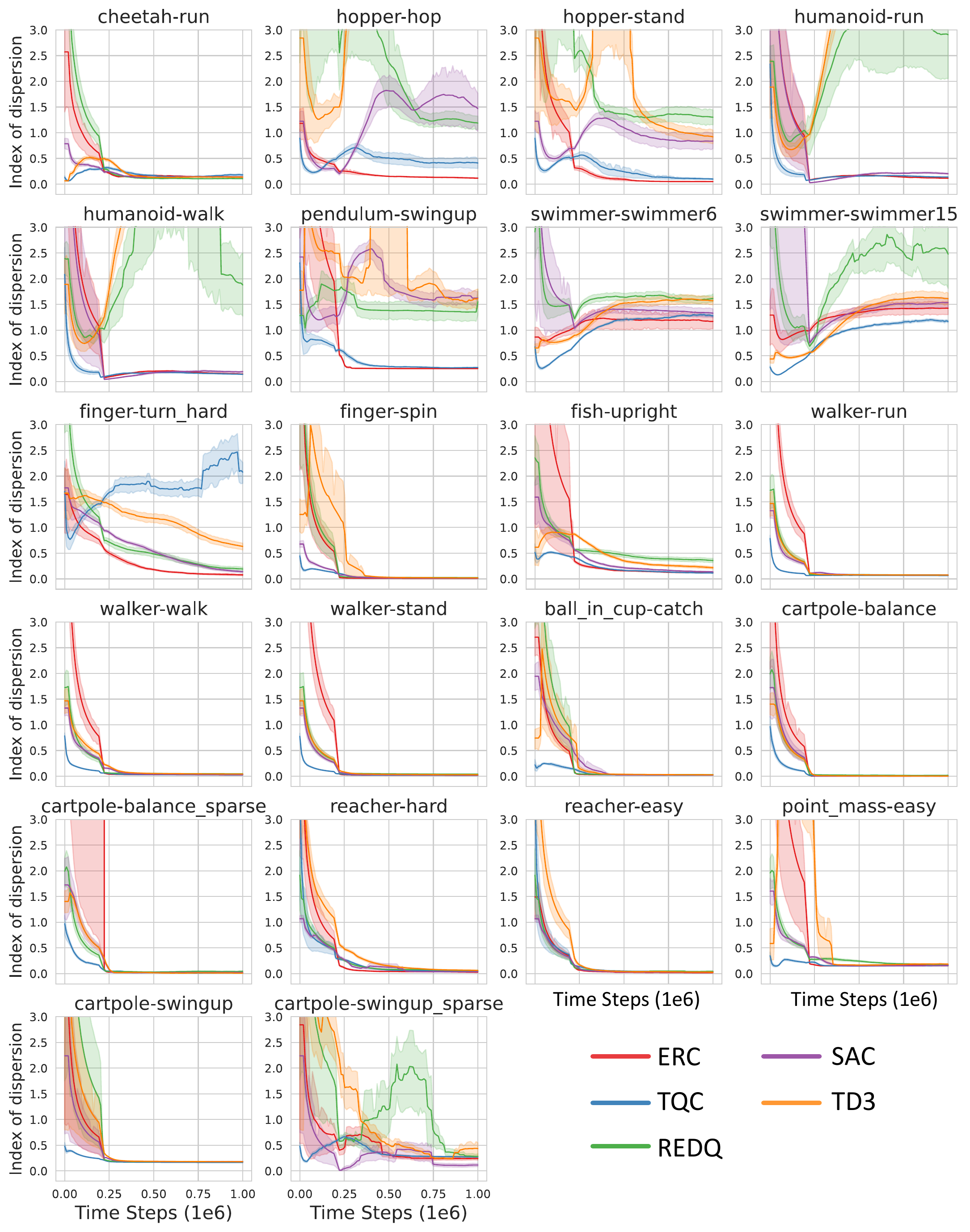}}
\vspace{-0.2in}
\caption{\label{app fig: dispersion 22}Index of dispersion curves. The results demonstrate that ERC effectively controls the variance. Furthermore, the variance of ERC is observed to be minimal on the selected tasks, even in situations where the performance of ERC is not as good as TQC and REDQ}
 \vspace{-0.2in}
\end{figure*}

\end{document}